\documentclass[9pt,twocolumn,twoside]{pnas-new}

\templatetype{pnasresearcharticle} 

\usepackage{times}
\usepackage{hyperref}
\usepackage{url}
\usepackage{graphicx}
\usepackage{multirow}

\usepackage{lipsum}
\usepackage{textcomp}
\usepackage{natbib}
\usepackage{fontawesome}
\usepackage{bbm}
\usepackage{listings}
\usepackage{textgreek} 
\usepackage{amsmath} 
\usepackage{amssymb,amsthm,mathtools}

\usepackage[capitalise]{cleveref}
\usepackage{stfloats} 
\usepackage{float}
\usepackage{placeins} 

\usepackage{tablefootnote} 
\usepackage[roman]{parnotes} 
\usepackage{tabulary}
\usepackage{booktabs}
\usepackage{tabularx}

\usepackage{algorithm}  
\usepackage{algorithmicx}
\usepackage[noend]{algpseudocode}
\definecolor{mygray}{gray}{0.5}
\definecolor{cblue}{RGB}{8, 85, 153}
\definecolor{darkblue}{RGB}{31, 64, 96}
\newcommand{\cblue}[1]{{\textcolor{cblue}{#1}}}
\definecolor{cgreen}{RGB}{8, 153, 83}
\definecolor{green}{RGB}{8, 200, 50}
\definecolor{cmaroon}{RGB}{128, 0, 0}
\algdef{SE}[DOWHILE]{Do}{doWhile}{\algorithmicdo}[1]{\algorithmicwhile\ #1}%

\renewcommand\algorithmicdo{}

\Crefname{section}{Sec}{Secs.}
\Crefname{figure}{Fig}{Figs.}
\Crefname{theorem}{Theorem}{Theorems.}

\newcommand{\var}[1]{\textcolor{darkblue}{\textit{#1}}}
\newcommand{\varm}[1]{\textcolor{cmaroon}{\textit{#1}}}

\usepackage{amsfonts}

\usepackage{tikzsymbols}

\newcommand{\method}{FIGS}
\newcommand{\methods}{FIGS }
\newcommand{\methodabbrv}{G-FIGS}

\newcommand{\bh}{\mathbf{h}}

\newcommand{\bx}{\mathbf{x}}
\newcommand{\by}{\mathbf{y}}

\newcommand{\br}{\mathbf{r}}
\newcommand{\bbeta}{\boldsymbol{\beta}}

\newcommand{\beps}{\boldsymbol{\epsilon}}
\newcommand{\beeta}{\boldsymbol{\eta}}

\newcommand{\btheta}{\boldsymbol{\theta}}

\newcommand{\bH}{\mathbf{H}}
\newcommand{\bW}{\mathbf{W}}
\newcommand{\bX}{\mathbf{X}}

\newcommand{\bPsi}{\boldsymbol{\Psi}}
\newcommand{\bSigma}{\boldsymbol{\Sigma}}

\newcommand{\R}{\mathbb{R}}
\renewcommand{\P}{\mathbb{P}}
\newcommand{\E}{\mathbb{E}}
\newcommand{\indicator}{\mathbf{1}}
\newcommand{\Var}{\textnormal{Var}}

\newcommand{\diag}{\textnormal{diag}}
\newcommand{\trace}{\textnormal{Tr}}
\newcommand{\grad}{\nabla}

\DeclarePairedDelimiter{\braces}{\lbrace}{\rbrace}
\DeclarePairedDelimiter{\paren}{(}{)}

\DeclarePairedDelimiter{\norm}{\|}{\|}
\DeclarePairedDelimiter{\inprod}{\langle}{\rangle}

\newcommand{\data}[1][n]{\mathcal{D}_{#1}}

\newcommand{\cell}{\mathcal{C}}
\newcommand{\tree}{\mathcal{T}}
\newcommand{\node}{\mathfrak{t}}

\newcommand{\uniform}{\mu}

\newcommand{\indep}{\perp \!\!\! \perp}

\newtheorem{theorem}{Theorem}[]
\newtheorem{lemma}[theorem]{Lemma}

\newtheorem{corollary}[theorem]{Corollary}

\newtheorem{remark}[theorem]{Remark}

\title{Fast Interpretable Greedy-Tree Sums}

\author[a,1]{Yan Shuo Tan}
\author[b,c,1]{Chandan Singh} 
\author[b,1]{Keyan Nasseri}
\author[d,1]{Abhineet Agarwal}
\author[e]{James Duncan}
\author[d]{Omer Ronen}
\author[f]{Matthew Epland}
\author[g,h]{Aaron Kornblith}
\author[b,c,d,2]{Bin Yu}

\affil[a]{Department of Statistics and Data Science, National University of Singapore}
\affil[b]{Department of Electrical Engineering and Computer Sciences, UC Berkeley}
\affil[c]{Microsoft Research}
\affil[d]{Department of Statistics, UC Berkeley}
\affil[e]{Graduate Group in Biostatistics, UC Berkeley}
\affil[f]{Overjet}
\affil[g]{Department of Emergency Medicine, UC San Francisco}
\affil[h]{Department of Pediatrics, UC San Francisco}

\leadauthor{Tan} 

\authorcontributions{}
\authordeclaration{The authors declare that no conflicts of interest exist.}
\equalauthors{\textsuperscript{1}Y.T.(Author One), C.S. (Author Two), K.N. (Author Three), A.A. (Author Four) contributed equally to this work.}
\correspondingauthor{\textsuperscript{2}To whom correspondence should be addressed: binyu@berkeley.edu}

\begin{abstract}

Modern machine learning has achieved impressive prediction performance, but often sacrifices interpretability, a critical consideration in high-stakes domains such as medicine. In such settings, practitioners often use highly interpretable decision tree models, but these suffer from inductive bias against additive structure. To overcome this bias, we propose Fast Interpretable Greedy-Tree Sums (FIGS), which generalizes the CART algorithm to simultaneously grow a flexible number of trees in summation. By combining logical rules with addition, FIGS is able to adapt to additive structure while remaining highly interpretable. Extensive experiments on real-world datasets show that FIGS achieves state-of-the-art prediction performance. To demonstrate the usefulness of FIGS in high-stakes domains, we adapt FIGS to learn clinical decision instruments (CDIs), which are tools for guiding clinical decision-making. Specifically, we introduce a variant of FIGS known as G-FIGS that accounts for the heterogeneity in medical data. G-FIGS derives CDIs that reflect domain knowledge and enjoy improved specificity (by up to 20\% over CART) without sacrificing sensitivity or interpretability. To provide further insight into FIGS, we prove that FIGS learns components of additive models, a property we refer to as disentanglement. Further, we show (under oracle conditions) that unconstrained tree-sum models leverage disentanglement to generalize more efficiently than single decision tree models when fitted to additive regression functions. Finally, to avoid  overfitting with an unconstrained number of splits, we develop Bagging-FIGS, an ensemble version of FIGS that borrows the variance reduction techniques of random forests. Bagging-FIGS enjoys competitive performance with random forests and XGBoost on real-world datasets.
    
\end{abstract}

\begin{document}

\maketitle
\ifthenelse{\boolean{shortarticle}}{\ifthenelse{\boolean{singlecolumn}}{\abscontentformatted}{\abscontent}}{}

\section{Introduction}
\label{sec:intro}

Modern machine learning methods such as random forests~\cite{breiman2001random}, gradient boosting~\cite{friedman2001greedy,chen2016xgboost}, and deep learning~\cite{lecun2015deep} display impressive predictive performance, but are complex and opaque, leading many to call them ``black-box'' models.
Model interpretability is critical in many applications~\cite{rudin2019stop,murdoch2019definitions}, particularly in high-stakes settings such as clinical decision instrument (CDI) modeling.
Interpretability allows models to be audited for general validation, errors, or biases, and therefore also more amenable to improvement by domain experts.
Interpretability also facilitates counterfactual reasoning, which is the foundation of scientific insight, and it instills trust/distrust in a model when warranted.
As an added benefit, interpretable models tend to be faster and more computationally efficient than black-box models.\footnote{\methods is integrated into the imodels package \href{https://github.com/csinva/imodels}{\faGithub\,github.com/csinva/imodels}~\cite{singh2021imodels} with an sklearn-compatible API. 
Experiments for reproducing the results here can be found at \href{https://github.com/Yu-Group/imodels-experiments}{\faGithub\,github.com/Yu-Group/imodels-experiments}.}

Decision trees are a prime example of interpretable models ~\cite{breiman1984classification,friedman2001greedy,quinlan2014c4,rudin2021interpretable,singh2021imodels}.
They can be easily visualized, memorized, and emulated by hand, even by non-experts, and thus fit naturally into high-stakes use-cases, such as decision-making in medicine (e.g., the emergency department\footnote{For example, in the popular tool \href{https://www.mdcalc.com/}{mdcalc}, over 90\% of available CDIs take the form of a decision tree.}), law, and public policy.
While decision trees have the potential to adapt to complex data, they are often outperformed by black-box models in terms of prediction performance.
However, there is evidence that this performance gap is not intrinsic to interpretable models, e.g., see examples in~\cite{rudin2021interpretable,ha2021adaptive,mignan2019one,singh2021imodels}.
In this paper, we identify an inductive bias of decision trees that causes its prediction performance to suffer in some instances, and design a new tree-based algorithm that overcomes this bias while preserving interpretability.

Our starting point is the observation that \emph{decision trees can be statistically inefficient at fitting regression functions with additive components}~\cite{tan2021cautionary}.
To illustrate this, consider the following toy example: $y = \mathbbm{1}_{X_1> 0} + \mathbbm{1}_{X_2 > 0} \cdot \mathbbm{1}_{X_3 > 0}$.\footnote{This toy model is an instance of a Local Spiky and Sparse (LSS) model~\cite{behr2021provable}, which is grounded in real biological mechanisms whereby an outcome is driven by interactions of inputs (e.g. bio-molecules) which display thresholding behavior.}
The two components of this function can be individually implemented by trees with 1 split and 2 splits respectively.
However, fitting their sum with a single tree requires at least 5 splits, as we are forced to make \emph{duplicate subtrees}:
a copy of the second tree has to be grown out of every leaf node of the first tree (see \cref{fig:intro}).
Indeed, given independent tree functions $f_1,\ldots,f_k$,
in order for a single tree $f$ to implement their sum, we would generally need
$$
\#\text{leaves}(f) \geq \prod_{k=1}^K \#\text{leaves}(f_k).
$$

This need to grow a deep tree to capture additive structure implies two statistical weaknesses of decision trees when fitting them to additive data-generating mechanisms.
First, growing a deep tree greatly increases the probability of splitting on noisy features.
Second, leaves in a deep tree contain fewer samples, implying that the predictions suffer from higher variance.
These weaknesses could be mitigated if we fit a separate tree to each additive component of the generative mechanism and present the \emph{tree-sum} as our fitted model.
While existing ensemble methods such as random forests \cite{breiman2001random} and gradient boosting \cite{friedman2001greedy,chen2016xgboost} comprise tree-sums, they either fit each tree independently (random forests) or sequentially (gradient boosting), and are hence unable to disentangle additive components.

\begin{figure*}[ht]
    \centering
    \includegraphics[width=0.8\textwidth]{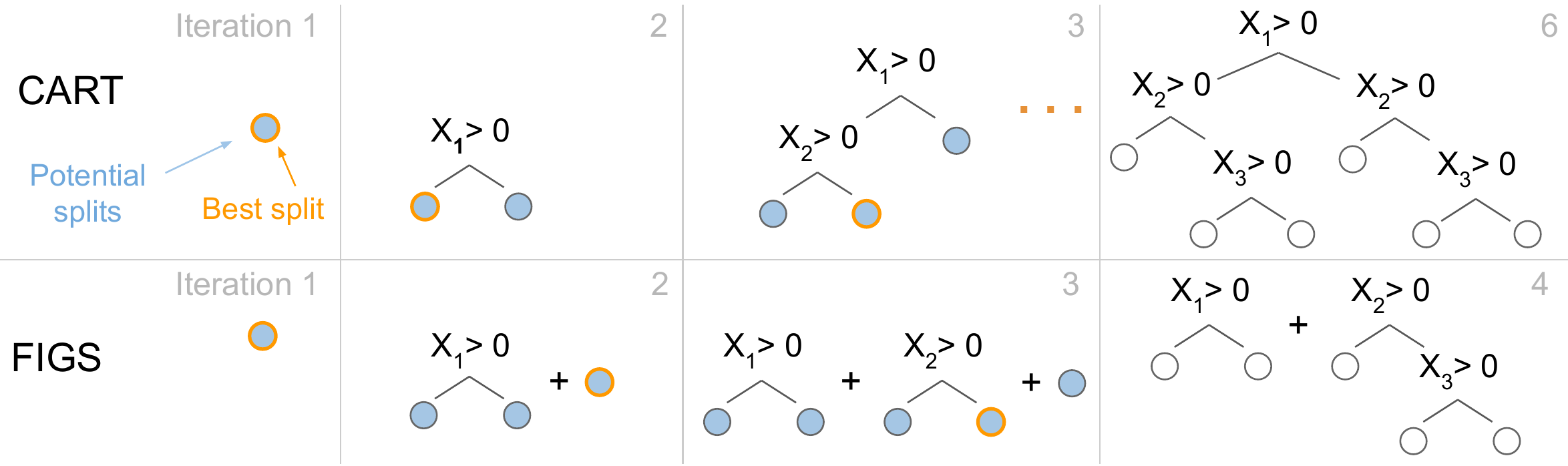}
    \caption{\methods algorithm overview for learning the toy function $y = \mathbbm{1}_{X_1> 0} + \mathbbm{1}_{X_2 > 0} \cdot \mathbbm{1}_{X_3 > 0}$. \methods greedily adds one node at a time, considering splits not just in an individual tree but within a collection of trees in a sum.
    This can lead to much more compact models, as it reduces repeated splits (e.g., in the final CART model shown in the top-right).}
    \label{fig:intro}
\end{figure*}

To address these weaknesses of decision trees, we propose Fast Interpretable Greedy-Tree Sums (\method), a novel yet natural algorithm that is able to \emph{grow a flexible number of trees simultaneously}.
FIGS is based on a simple yet effective modification to Classification and Regression Trees (CART)~\cite{breiman1984classification}, allowing it to adapt to additive structure (if present) by starting new trees, while still maintaining the ability of CART to adapt to higher-order interaction terms.
By capping the total number of splits allowed, FIGS produces a model that is also easily visualized, memorized, and emulated by hand.

We performed extensive experiments across a wide array of real-world datasets to compare the predictive performance of \methods to a number of popular decision rule models \cite{molnar2020interpretable}. Specifically, we took the number of rules (splits in the case of trees) as a common measure of interpretability for this model class, and constructed decision rule models at a prescribed level of interpretability. Our results show that \methods often achieved the best predictive performance  across various levels of interpretability (i.e., number of splits). 

Next, we apply \methods to a key domain problem involving decision rule models which is the construction of CDIs. CDIs are models for predicting patient risk and are widely used by healthcare professionals to make rapid decisions in a clinical setting such as the emergency room. To apply \methods to the medical domain, it is crucial that \methods  is able to adapt to heterogeneous data from diverse groups of patients. Tailoring models to various groups is often necessary since various groups of patients may differ dramatically and require distinct features for high predictive performance on the same outcome. While a naive solution is to fit a unique model to each group of patients, this is sample-inefficient and sacrifices valuable information that can be shared across groups. To mitigate this issue, we introduce a variant of \method, called Group Probability-Weighted Tree Sums (G-\method), which accounts for this heterogeneity while using all the samples available. Specifically, G-\methods first fits a classifier (e.g., logistic regression) to predict group membership probabilities for each sample. Then, it uses these estimates as instance weights in FIGS to output a model for each group. 


We used both \methods and G-\methods to construct CDIs for three pediatric emergency care datasets, and found that both have up to 20\% higher specificity than CART while maintaining the same level of sensitivity. Further, G-\methods improves specificity over \methods by over 3\% for fixed levels of sensitivities. The features used in the fitted models agrees with medical domain knowledge. Moreover, G-\methods learns a tree-sum model where each tree in the model represents a distinct clinical domain, providing a clear and organized framework for clinicians to use when assessing and treating patients. Next, we investigate the stability of G-\methods to data perturbations. Stability to ``reasonable'' data perturbations is a key tenet of the Predictability, Computability, and Stability (PCS) framework \cite{yu2013stability,yu2020veridical,murdoch2019definitions}, and a necessary prerequisite for the application of ML techniques in high-stakes domains. We show that G-\methods learns a similar tree-sum model (i.e., a similar set of features) across data perturbations (e.g., introducing noise by randomly permuting labels). This work expands on initial results on G-\methods contained in an earlier pre-print \cite{nasseri2022group}. 

Next, to provide insight into the success of \method, we investigate tree-sum models and \method~theoretically. We prove generalization upper bounds for tree-sum models when given oracle access to the optimal tree structures. Whenever additive structure is present, this upper bound has a faster rate in the sample size $n$ compared to the generalization lower bound for any decision tree proved in \cite{tan2021cautionary}. Further, we establish in the large-sample limit that \method~disentangles the additive components of the generative model, with each tree in the sum fitting a separate additive component.

Lastly, similarly to CART, \method~overfits when allowed too many splits. Hence, we develop an ensemble version, called Bagging-FIGS, that borrows the bootstrap and feature subsetting strategies of random forests.
We compare the prediction performance of Bagging-FIGS against random forests, XGBoost, and
generalized additive models (GAMs) across a wide range of real-world datasets. Bagging-FIGS always maintains competitive performance with all other methods, further enjoying the best performance on several datasets.



In what follows,
\cref{sec:methods} introduces \methods and
\cref{sec:background} covers related work.
\cref{sec:results} contains our experimental results on real-world datasets showing that \methods predicts well with very few splits.
\cref{sec:cdr_results} covers three CDI case studies using \methods and G-\method.
\cref{sec:theoretical_investigations} investigates the theoretical performance of tree-sum models and \method. Finally, \cref{sec:bagging_figs} introduces Bagging-\methods and compares its prediction performance to other algorithms on real-world datasets. 




\section{\method: Algorithm description and run-time}
\label{sec:methods}


Suppose we are given training data $\data = \braces*{(\bx_{i},y_{i})}_{i=1}^n$. 
When growing a tree, CART chooses for each node $\node$ the split $s$ that maximizes the impurity decrease in the responses $\by$.
For a given node $\node$, the impurity decrease has the expression \begin{align*}
    \hat\Delta(s,\node, \by) \coloneqq & \sum_{\bx_i \in \node} \paren*{y_i - \bar y_\node}^2 
    - \sum_{\bx_i \in \node_L} \paren*{y_i - \bar y_{\node_L}}^2 \\
    &
    - \sum_{\bx_i \in \node_R} \paren*{y_i - \bar y_{\node_R}}^2,
\end{align*}
where $\node_L$ and $\node_R$ denote the left and right child nodes of $\node$ respectively, and $\bar y_{\node}, \bar y_{\node_L}, \bar y_{\node_R}$ denote the mean responses in each of the nodes.
We call such a split $s$ a \emph{potential split}, and note that for each iteration of the algorithm, CART chooses the potential split with the largest impurity decrease.\footnote{This corresponds to greedily minimizing the mean-squared-error criterion in regression and Gini impurity in classification.}

\methods extends CART to greedily grow a tree-sum (see \cref{alg:method}).
That is, at each iteration of FIGS, the algorithm chooses either to make a split on one of the current $K$ trees $\hat f_1,\ldots, \hat f_K$ in the sum, or to add a new stump to the sum.
To make this decision, it still applies the CART splitting criterion detailed above to identify a potential split in each leaf of each tree.
However, to compute the impurity decrease for a given split, it substitutes the vector of \emph{residuals} $\textcolor{cmaroon}{r_i} \coloneqq y - \sum_{k=1}^K \hat f_k(\bx_i)$ for the vector of responses $\by$.
FIGS makes only one split among the $K+1$ trees:
the one corresponding to the largest impurity decrease.
The value of each of the new leaf nodes is then defined to be the mean residual for the samples it contains, added to the value of its parent node.
If a new stump is created, the value at the root is defined to be zero.
At inference time, we predict the response of an example by dropping it down each tree and summing the values of each leaf node containing it.

\begin{algorithm}[h]
  \caption{\methods fitting algorithm.}
  \label{alg:method}
  \small
\begin{algorithmic}[1]
  \State {\bfseries \method}(\var X: features, \var y: outcomes, \var {stopping\_threshold}) 

  \State \var{trees} = []
  \While{count\_total\_splits(\var{trees}) $<$ \var {stopping\_threshold}:}
    \State \var{all\_trees} = join(\var{trees}, build\_new\_tree()) \textcolor{gray}{\# add new tree}
    \State \var{potential\_splits} = []
    \For{}\hspace{-3pt}\var{tree} in \var{all\_trees}:
        \State \varm{y\_residuals} = \var{y} -- predict(\var{all\_trees}) 
        \For{}\hspace{-3pt}\var{leaf} in \var{tree}:
            \State \var{potential\_split} = split(\var{X}, \varm{y\_residuals}, \var{leaf})
            \State \var{potential\_splits}.append(\var{potential\_split})
        \EndFor
    \EndFor
    \State \var{best\_split} = split\_with\_min\_impurity(\var{potential\_splits})
    \State \var{trees}.insert(\var{best\_split})
  \EndWhile
\end{algorithmic}
\end{algorithm}

\paragraph{Selecting the model's stopping threshold.}Choosing a threshold on the total number of splits can be done similarly to CART:
using a combination of the model's predictive performance (i.e., cross-validation) and domain knowledge on how interpretable the model needs to be.
Alternatively, the threshold can be selected using an impurity decrease threshold~\cite{breiman1984classification} rather than a hard threshold on the number of splits.
We discuss potential data-driven choices of the threshold in the Discussion (\cref{sec:discussion}).

\paragraph{Run-time analysis.} The run-time complexity for \method~to grow a model with $m$ splits in total is $O(dm^2n^2)$, where $d$ the number of features, and $n$ the number of samples (see derivation in \cref{sec:run_time_extensions_supp}). In contrast, CART has a run-time of $O(dmn^2)$.
Both of these worst-case run-times given above are quite fast, and the gap between them is relatively benign as we usually make a small number of splits to ensure interpretability.

\paragraph{Extensions.}
\methods supports many natural modifications that are used in CART trees.
For example, different impurity measures can be used; here we use Gini impurity for classification and mean-squared-error for regression.
Additionally, \methods could benefit from pruning, shrinkage \cite{agarwal2022hierarchical}, or by being used as part of an ensemble model (e.g., Bagging-\method). We discuss other extensions in \cref{sec:run_time_extensions_supp}. 

\section{Related work and its connections to FIGS}
\label{sec:background}

There is a long history of greedy methods for learning individual trees, e.g., C4.5~\cite{quinlan2014c4}, CART~\cite{breiman1984classification}, and ID3~\cite{quinlan1986induction}.
Recent work has proposed global optimization procedures rather than greedy algorithms for trees.
These can improve performance given a fixed split budget but incur a high computational cost ~\cite{lin2020generalized,hu2019optimal,bertsimas2017optimal}.
However, due to the limitations of a single tree, all these methods have an inductive bias against additive structure \cite{bagallo1990boolean,tan2021cautionary}. Besides trees, there are a variety of other interpretable methods such as rule lists~\cite{letham2015interpretable,angelino2017learning}, rule sets~\cite{cohen1999simple,dembczynski2008maximum}, or generalized additive models~\cite{caruana2015intelligible}; for an overview and Python implementation, see~\cite{singh2021imodels}.

\methods is related to backfitting \cite{breiman1985estimating}, but differs crucially because it does not assume a fixed number of component features, nor does it require knowledge on which features are used by each component.
Furthermore, \methods does not update its component trees in a cyclic manner, but instead  trees ``compete for splits'' at each iteration.

Similar to the work here are methods that learn an additive model of splits, where a split is defined to be an axis-aligned, rectangular region in the input space.
RuleFit~\cite{friedman2008predictive} is a popular method that learns a model by first extracting splits from multiple greedy decision trees fit to the data and then learning a linear model using those splits as features.
\methods is able to improve upon RuleFit by greedily selecting higher-order interactions when needed, rather than simply using all splits from some pre-specified tree depth.
MARS~\cite{friedman1991multivariate} greedily learns an additive model of splines in a manner similar to \method, but loses some interpretability as a result of using splines rather than splits.

Also related to this work are tree ensembles, such as random forest~\cite{breiman2001random}, gradient-boosted trees~\cite{freund1996experiments}, BART~\cite{chipman2010bart} and AddTree~\cite{luna2019building}, all of which use ensembling as a way to boost predictive accuracy without focusing on finding an interpretable model.
Loosely related are post-hoc methods which aim to help understand a black-box model~\cite{lundberg2019explainable,friedman2001greedy,devlin2019disentangled,agarwal2023mdi}, but these can display lack of fidelity to the original model, and also suffer from other problems \cite{rudin2018please}.

\section{\methods results on real-world benchmark datasets}\label{sec:results}

This section shows that \methods enjoys strong prediction performance on several real-world benchmark datasets compared to popular algorithms for fitting decision rule models.

\paragraph{Benchmark datasets.}
For classification, we study four large datasets previously used to evaluate rule-based models~\cite{pmlr-v97-wang19a,agarwal2022hierarchical} along with the two largest UCI binary classification datasets used in Breiman's original paper introducing random forests~\cite{breiman2001random,asuncion2007uci} (overview in \cref{tab:datasets}).
For regression, we study all datasets used in the random forest paper with at least 200 samples along with three of the largest non-redundant datasets from the PMLB benchmark~\cite{romano2020pmlb}. 
80\% of the data is used for training/3-fold cross-validation and 20\% of the data is used for testing.

\begin{table}[h]
    \centering
              

\begin{tabular}{clrrr}
 \toprule
  &                                       Name &  Samples &  Features & Majority class\\
 \midrule
 \parbox[c]{2mm}{\multirow{6}{*}{\rotatebox[origin=c]{90}{Classification}}}               &                    Readmission &   101763 &       150 & 53.9\%\\
& Credit \cite{yeh2009comparisons} &    30000 &        33 & 77.9\%\\
&                    Recidivism &     6172 &        20 & 51.6\%\\
            & Juvenile \cite{osofsky1997effects} &     3640 &       286 & 86.6\%\\
 & German credit &     1000 &        20 & 70.0\%\\             
& Diabetes \cite{smith1988using} &      768 &         8 & 65.1\%\\
\midrule
&                                       Name &  Samples &  Features & Mean\\
\midrule
  \parbox[c]{0.5mm}{\multirow{9}{*}{\rotatebox[origin=c]{90}{Regression}}} &     Breast tumor \cite{romano2020pmlb} &   116640 &         9 & 62.0\\
& CA housing \cite{pace1997sparse} &    20640 &         8 & 5.0\\
  &      Echo months \cite{romano2020pmlb} &    17496 &         9 & 74.6\\
  &  Satellite image \cite{romano2020pmlb} &     6435 &        36 & 7.0\\
  &      Abalone \cite{nash1994population} &     4177 &         8 & 29.0\\
  &    Diabetes  \cite{efron2004least} &      442 &        10 & 346.0\\
& Friedman1 \cite{friedman1991multivariate} &      200 &        10 & 26.5\\
& Friedman2 \cite{friedman1991multivariate} &      200 &         4 & 1657.0\\
& Friedman3 \cite{friedman1991multivariate} &      200 &         4 & 1.6\\
 \bottomrule
 \end{tabular}
 
    \caption{Real-world datasets analyzed here: classification (top panel), regression (bottom panel).}
    \label{tab:datasets}
\end{table}

\definecolor{bluegray}{rgb}{0.4, 0.6, 0.8}
\definecolor{darkgreen}{rgb}{0.0, 0.2, 0.13}


\paragraph{Baseline methods.}
For both classification and regression, \textbf{\method} is compared to \textcolor{orange}{CART}, \textcolor{darkgreen}{RuleFit}, and \textcolor{purple}{Boosted Stumps} (CART stumps learned via gradient-boosting).
Furthermore, for classification we additionally compare against \textcolor{bluegray}{C4.5} and for regression we additionally compare against a CART model fitted using the mean-absolute-error \cblue{(MAE)} splitting-criterion.
We finally also add a black-box baseline comprising a \textcolor{gray}{Random Forest} with 100 trees, which thus uses many more splits than all the other models.\footnote{We also compare against Gradient-boosting with decision trees of depth 2, but find that it is outperformed by CART in this limited-split regime, so we omit these results for clarity. 
We also attempt to compare to optimal tree methods, such as GOSDT~\cite{lin2020generalized}, but find that they are unable to fit the dataset sizes here.}


\paragraph{\methods predicts well with few splits.}
\cref{fig:performance_curves} shows the models' performance results (on test data) as a function of the number of splits in the fitted model.\footnote{For RuleFit, each term in the linear model is counted as one split.}
We treat the number of splits as a common metric for the level of interpretability of each model.
In practice, interpretability is context-dependent. However, the goal of this experiment is to sweep across a range of interpretability levels, and show the test performance  had that level been selected. 

The top two rows of \cref{fig:performance_curves} show results for classification (measured using the area under the receiver operating characteristic (ROC) curve, i.e., AUC), and the bottom three rows show results for regression (measured using the coefficient of determination, denoted by $R^2$).
On average, \methods outperforms baseline models when the number of splits is very low.
The performance gain from \methods over other baselines is larger for the datasets with more samples (e.g., the top row of \cref{fig:performance_curves}), matching the intuition that \methods performs better because of its increased flexibility.
For two of the larger datasets (\textit{Credit} and \textit{Recidivism}), \methods even outperforms the black-box \textcolor{gray}{Random Forest} baseline, despite using less than 15 splits.
For the smallest classification dataset (\textit{Diabetes}), \methods performs extremely well with very few (less than 10) splits but starts to overfit as more splits are added.

\begin{figure*}
    \centering
    \begin{tabular}{lc}
    \begin{minipage}{0.04cm}
    \rotatebox[origin=c]{90}{\large Classification}
    \end{minipage}%
         & \raisebox{\dimexpr-.5\height-1em}{\includegraphics[width=0.9\textwidth]{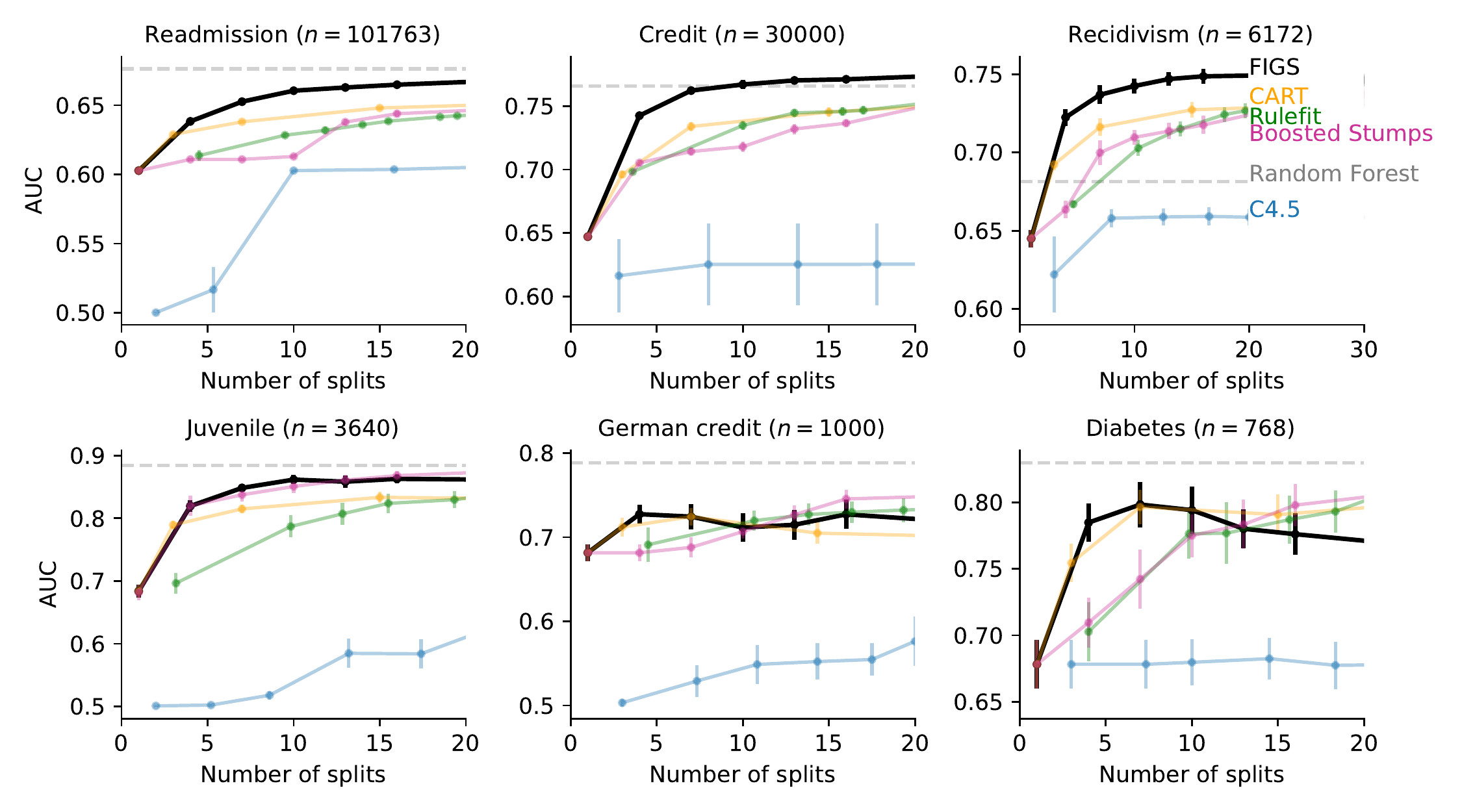}} \\
        \midrule
    \begin{minipage}{0.04cm}
    \rotatebox[origin=c]{90}{\large Regression}
    \end{minipage}%
        & \raisebox{\dimexpr-.5\height-1em}{\includegraphics[width=0.9\textwidth]{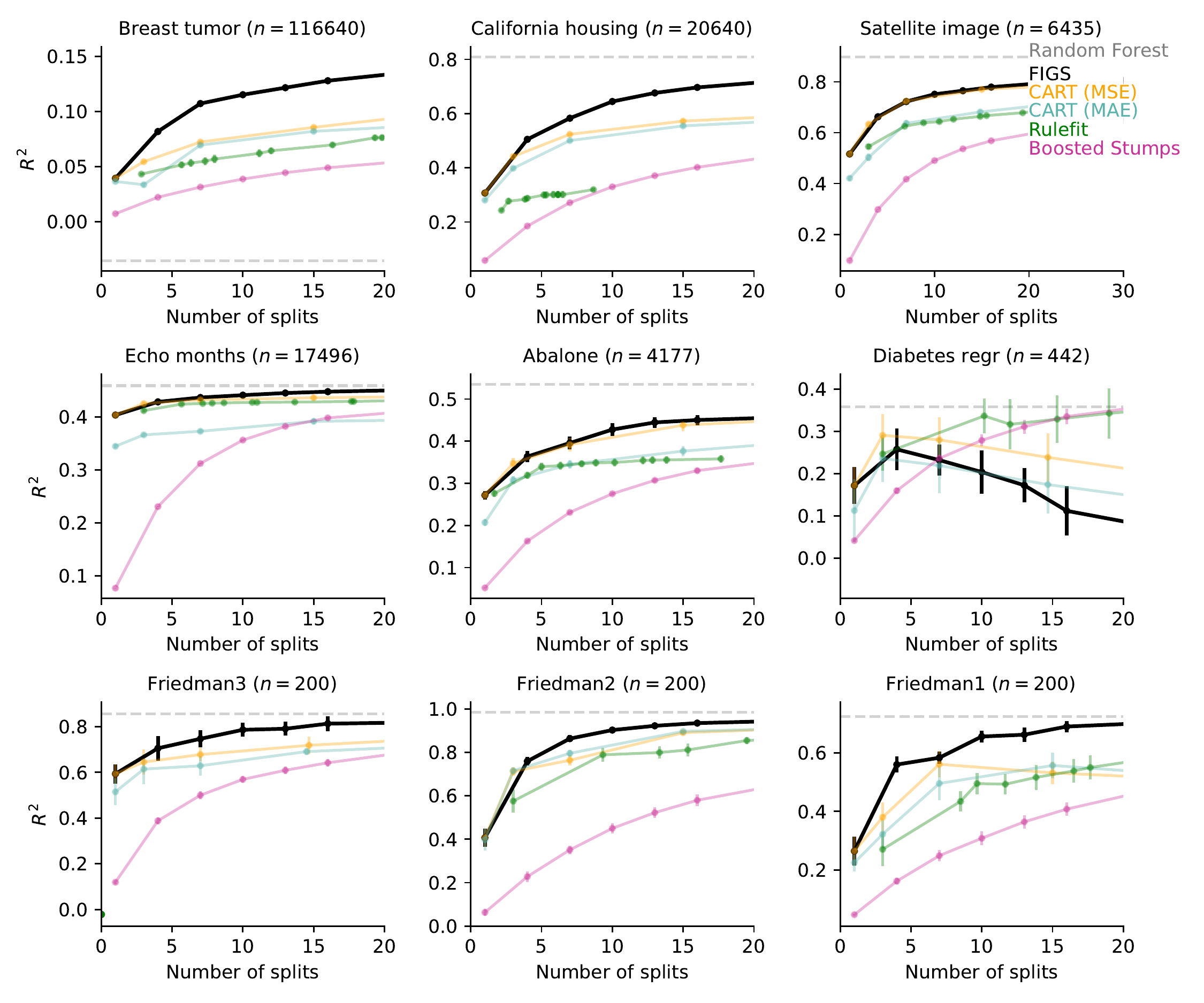}}
    \end{tabular}
    \vspace{-10pt}
    \caption{\methods performs extremely well on the test-set using very few splits, particularly when the dataset is large.
    Top two rows show results for classification datasets (measured by AUC of the ROC curve) and the bottom three rows show results for regression datasets (measured by $R^2$).
    Errors bars show standard error of the mean, computed over 6 random data splits.}
    \label{fig:performance_curves}
\end{figure*}

\section{Learning CDIs via \methods}
\label{sec:cdr_results}

A key domain problem involving interpretable models is the development of CDIs, which can assist clinicians in improving the accuracy, consistency, and efficiency of diagnostic strategies for sick and injured patients.
Recent works have developed and validated CDIs using interpretable models, particularly in emergency medicine~\cite{bertsimas2019prediction,stiell2001canadian,kornblith2022predictability,holmes2002identification}. As discussed earlier, applying machine learning models to the medical domain must account for the heterogeneity in the data that arises from the presence of diverse groups of patients (e.g., age groups, sex, treatment sites). Typically, this heterogeneity is dealt by fitting a separate model on each group. However, this strategy comes at the cost of losing samples, and discarding valuable information that can be shared amongst various groups. To mitigate this loss of power, we introduce a variant of \method, called group-probability weighted FIGS, or G-FIGS as follows. 


\paragraph{G-FIGS: Fitting \methods to multiple groups.}
As before, we assume a supervised learning setting with features $X$, outcome $Y$, and a group label $G$ (e.g., treatment site or age group). G-FIGS is a two-step algorithm: (i) For a given group label $g$, G-FIGS first estimates group membership probabilities for each sample (e.g., by fitting a logistic regression model to predict group-membership)\footnote{This is methodologically analogous to a propensity score in the causal inference literature.} That is, it estimates $\mathbb{P}(G = g ~|~ X)$. (ii) For a given group $g$, G-FIGS then uses the group probabilities $\mathbb{P}(G = g ~|~ X)$ as sample weights when fitting \methods to the \emph{whole dataset}. This results in a group-specific model, but borrows information from samples in other groups. See~\cref{sec:g_figs_supp} for more details, and a visual representation.

\paragraph{Datasets and data cleaning.}
\cref{tab:cdr_datasets} shows the CDI datasets under consideration here.
They each constitute a large-scale multi-site data aggregation by the Pediatric Emergency Care Applied Research Network (PECARN), with a relevant clinical outcome (e.g., presence of traumatic brain injury).
For each of these datasets, we group patients into two natural groups: patients with age ${<}2$ years and ${\ge}2$ years.
This age-based threshold is commonly used for emergency-based diagnostic strategies \cite{kuppermann2009identification}, because it follows a natural stage of development, including a child’s ability to participate in their care (e.g., ability to verbally communicate with their doctor).
At the same time, the natural variability in early childhood development also creates opportunities to share information across this threshold.
These datasets are non-standard for machine learning; as such, we spend considerable time cleaning, curating, and preprocessing these features along with medical expertise included in the authorship team.\footnote{Details, along with the openly released clean data can be found in  \cref{sec:cdi_results_supp}.}
We use 60\% of the data for training, 20\% for tuning hyperparameters (including estimation of each patient's group-membership $\P(\text{age} \ {<}2 \ \text{years} ~| ~ X)$), and 20\% for evaluating test performance. 

\paragraph{Prediction metrics.} Prediction performance is measured by comparing the specificity of a model when sensitivity is constrained to be above a given threshold, chosen to be $\{92\%,94\%,96\%,98\% \}$. We opt for this metric because high levels of sensitivity are crucial for CDIs so as to avoid potentially life threatening false negatives (i.e., missing a diagnosis). For a given threshold, we would like to maximize specificity as false positives lead to unncessary resource utilization and can needlessly expose patients to the harmful effects of medical procedures such as radiation from a computed tomography (CT) scan.

\paragraph{Baseline methods.} We compare \methods and \methodabbrv~to two baselines: CART~\cite{breiman1984classification} and Tree-Alternating Optimization TAO~\cite{carreira2018alternating}).
For each baseline, we either (i) fit one model to all the training data or (ii) fit a separate model to each group (denoted with -SEP) -- one to the patients with age ${<}2$ years and one for the patients with age ${\ge}2$ years. Additionally, for CART, we also fit a model in the style of \methodabbrv, denoted as G-CART. Limits on the total number of splits for each model are varied over a range which yields interpretable models, from 2 to 16 maximum splits\footnote{The choice of 16 splits is somewhat arbitrary, but we find that amongst 643 popular CDIs on \href{https://www.mdcalc.com/}{mdcalc}, 93\% contain no more than 16 splits and 95\% contain no more than 20 splits.} (full details of this and selection of other hyperparameters are in\cref{sec:cdi_results_supp}).

\begin{table}[h]
    \centering
    \begin{tabular}{lrrrr}
\toprule
Name &  Patients &  Features & Outcome (count)&  Outcome (\%)  \\
\midrule
 TBI &     42428&        61 &      376 &        0.9\%  \\
 IAI &     12044 &        21 &      203 &        1.7\% \\
 CSI &      3313 &        34 &      540 &       16.3\% \\
\bottomrule
\end{tabular}
    \caption{Clinical decision datasets for traumatic brain injury (TBI)~\cite{kuppermann2009identification}, intra-abdominal injury (IAI)~\cite{holmes2002identification}, and cervical spine injury (CSI)~\cite{leonard2019cervical}.}
    \label{tab:cdr_datasets}
\end{table}


\paragraph{\methods and \methodabbrv~predict well.}

\cref{tab:results} shows the prediction performance of \method, \methodabbrv~, and baseline methods. Further, we report the prediction performance of all methods for each age group separately (i.e, age ${<}2$ years and age ${\geq}2$ years) in \cref{tab:cdr_results_young} and \cref{tab:cdr_results_old} respectively. For high levels of sensitivity, \methodabbrv~generally improves the model's specificity against the baselines. Further, \methodabbrv~also improves specificity for each age group, often outperforming both the models that fit all the data as well as the model that fits data for each group separately. This suggests that 
using a different model for each group while sharing information across groups can lead to better prediction in a heterogeneous population.

\begin{figure}[ht]
    \centering
    \includegraphics[width=\columnwidth]{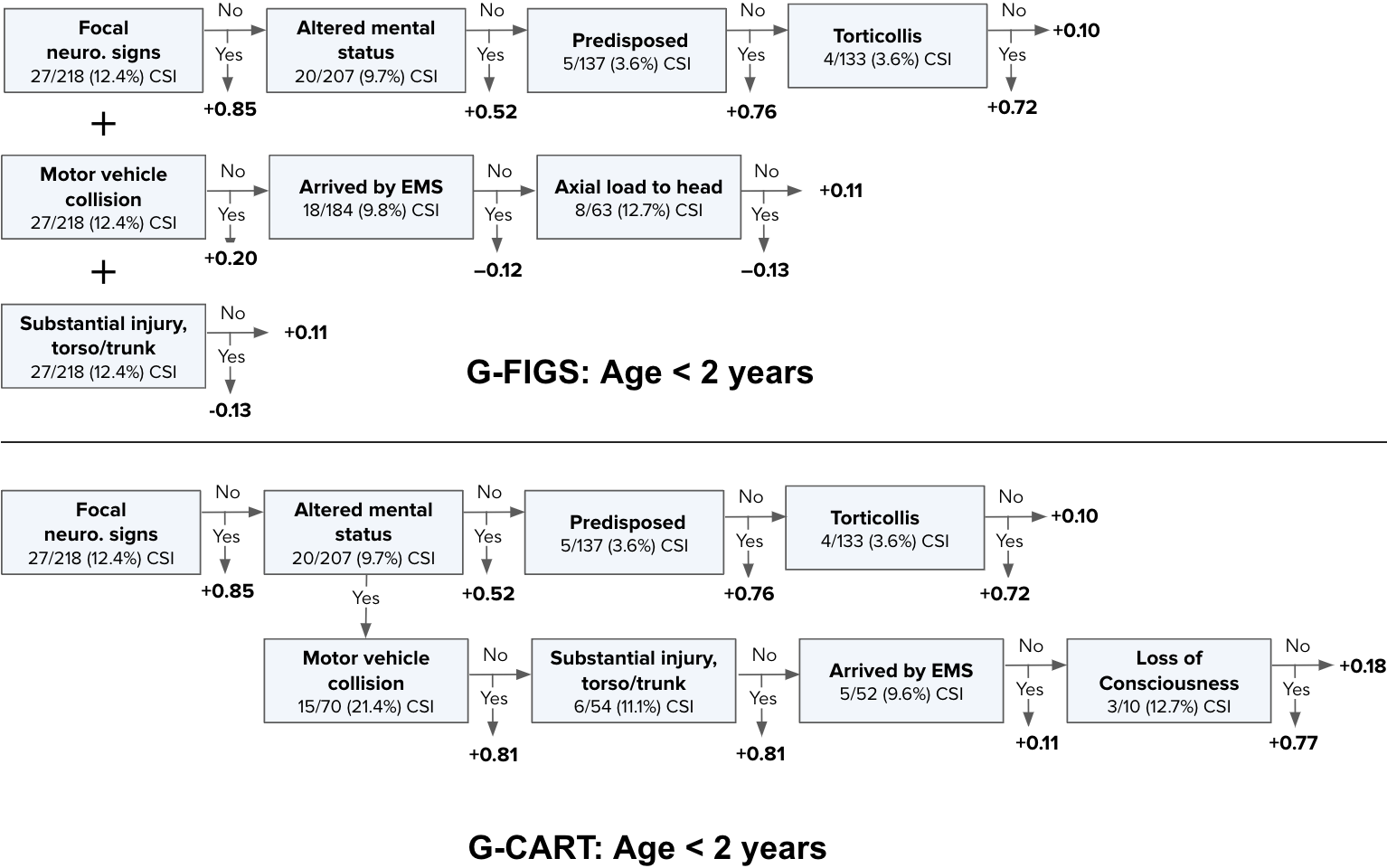}\\
    \caption{Visualizes the  \methodabbrv~and G-CART models fitted to CSI dataset for the age ${\le}2$ years group. Both \methodabbrv~and G-CART use features that match medical knowledge. However, unlike G-CART,  \methodabbrv~disentangles risk factors, with each tree representing a distinct clinical domain.}
    
    \label{fig:model_ex_csi}
\end{figure}

\begin{table*}[ht]
    \centering
        \begin{tabular}{lrrrrrrrrrrrr}
    \toprule
    {} & \multicolumn{4}{c}{Traumatic brain injury} &
        \multicolumn{4}{c}{Cervical spine injury} 
        & \multicolumn{4}{c}{Intra-abdominal injury} \\
     \cmidrule(lr){2-5}
     \cmidrule(lr){6-9}
     \cmidrule(lr){10-13}
     Sensitivity level: & 92\% & 94\% & 96\% & 98\% & 92\% & 94\% & 96\% & 98\% & 92\% & 94\% & 96\% & 98\% \\
     \midrule
    TAO & 6.2 & 6.2 & 0.4 & 0.4
        & 41.5 & 21.2 & 0.2 & 0.2 
        & 0.2  & 0.2 & 0.0 & 0.0 \\
     TAO-SEP & 26.7 & 13.9  & 10.4 & 2.4 
        & 32.5  & 7.0  & 5.4 & 2.5 
        & 12.1  & 8.5 & 2.0 & 0.0 \\
    CART & 20.9 & 14.8 & 7.8 & 2.1 
        & 38.6 & 13.7 & 1.5 & 1.1
        & 11.8 & 2.7 & 1.6 & 1.4 \\
     CART-SEP & 26.6 & 13.8 & 10.3 &  2.4 
        & 32.1 & 7.8 & 5.4 & 2.5 
        & 11.0 & 9.3 & 2.8 & 0.0 \\
    G-CART & 15.5 &  13.5 & 6.4 & 3.0
        &  38.5 &  15.2 & 4.9 & 3.9 
        & 11.7 &  10.1 & 3.8 &  0.7 \\
     \textbf{FIGS} & 23.8 &  18.2 &  12.1 &  0.4 
        & 39.1 & 33.8 & 24.2 & \textbf{16.7}   
        & \textbf{32.1} & 13.7 & 1.4 & 0.0  \\
     \textbf{FIGS-SEP} & 39.9 & 19.7&  \textbf{17.5} &  2.6 
        & 38.7 & 33.1 & 20.1 & 3.9 
        & 18.8 & 9.2 & 2.6 &  0.9  \\
     \textbf{\methodabbrv} & \textbf{42.0} & \textbf{23.0} & 14.7 & \textbf{6.4} 
        & \textbf{42.2}  & \textbf{36.2} & \textbf{28.4} & 15.7 
        & 29.7 & \textbf{18.8} & \textbf{11.7} & \textbf{3.0} \\
    \bottomrule
    \end{tabular}

        \caption{Best test-set specificity when sensitivity is constrained to be above a given level. \methodabbrv~provides the best performance overall in the high-sensitivity regime.
        -SEP models fit a separate model to each group, and generally outperform fitting a model to the entire dataset.
        G-CART follows the same approach as \methodabbrv~but uses weighted CART instead of FIGS for each final group model.
        Averaged over 10 random data splits into training, validation, and test sets, with hyperparameters chosen independently for each split.
        }
    \label{tab:results}
\end{table*}


\paragraph{Features used by \methodabbrv~match medical knowledge.} 
\cref{fig:model_ex_csi} shows the \methodabbrv~model on the cervical spinal injury (CSI) dataset for patients with age ${<}2$ years, while \cref{sec:cdi_results_supp} shows \methodabbrv~models for patients with age ${\geq}2$ years, and the other two datasets. 
The features used by the learned model match medical domain knowledge and partially agree with previous work~\cite{leonard2019cervical}; e.g., features such as \textit{focal neurologic signs}, \textit{altered mental status}, and  \textit{torticollis} are all known to increase the risk of CSI. Further, features unique to each group largely relate to the age cutoff; the ${<}2$ years age group features include those that clinicians can assess without asking the patient (e.g., \textit{substantial torso injury}), while that of the  ${\ge}2$ years age group features (visualized in \cref{fig:model_ex_csi_rep}) require verbal responses (\textit{neck pain}, \textit{head pain}). This matches the medical intuition that non-verbal features should be more reliable features in the ${<}2$ years age group.


\paragraph{G-FIGS disentangles clinical risk factors.} \cref{fig:model_ex_csi} visualizes the \methodabbrv~and G-CART model for the ${<}2$ years age group on the CSI dataset. Both \methodabbrv~and G-CART utilize almost all the same features (apart from \textit{Axial load to head}). However, unlike G-CART, \methodabbrv~disentangles risk factors with each tree representing a clinical domain: the top tree identifies signs and symptoms, the middle tree corresponds to the mode of injury and how the patient arrived at the hospital, and the bottom tree assesses the overall severity of the patient's injury patterns and any associated injuries that may be present. 
By disentangling clinical domains, the fitted \methodabbrv~model not only has stronger prediction performance than a single-tree model (see \cref{tab:cdr_datasets}), but also provides a more medically intuitive framework for clinicians to use when assessing and treating patients. We provide another example of the ability of \method~to disentangle risk factors on a diabetes classification dataset in \cref{sec:theory_supp}. We note that there is no a priori reason that tree-sum models are more reflective of the true data generating process than single tree models. Instead, trust in the veracity of such a model should be based on good prediction performance as well as coherence with domain knowledge. In more detail, we recommend that practitioners follow the predictive, descriptive, and relevance (PDR) framework for investigating real-world scientific problems using interpretable machine learning models \cite{murdoch2019definitions}. 



\paragraph{Stability Analysis of Learnt CDI.} As discussed earlier, the PCS framework argues that stability to ``reasonable'' data perturbations is a key prerequisite for interpretability and deployment of machine learning methods in high-stakes domains. Here, we investigate the stability of \methodabbrv~on the CSI dataset. Specifically, we introduce noise by randomly swapping a percentage $p$ of labels $\by$. We vary $p$ between $\{1\%,2.5\%,5\%\}$. For each value of $p$, we measure stability by comparing the similarity of the features selected in the model trained on the perturbed data to the model displayed in \cref{fig:model_ex_csi}. In particular, the similarity of features is measured via the Jaccard distance. Further details of our experiments, and our results can be found in \cref{sec:cdi_results_supp}. Our results show that \methodabbrv~learns a similar model for each age group even for larger values of $p$, indicating its stability.

\section{Theoretical Investigations}
\label{sec:theoretical_investigations}
We perform theoretical investigations to better understand the properties of \method~and tree-sum models. Specifically, we show (under some oracle conditions) that if the regression function has an additive decomposition, then tree-sum models and \methods are able to achieve optimal generalization upper bounds and disentangle additive components. In this section, we summarize these theoretical results, and defer the formal statements and proof details to \cref{sec:theory_supp}.


\paragraph{Oracle generalization upper bounds.}
As discussed, \emph{all} single-tree models have a squared error generalization error lower bound of $\Omega(n^{-2/(d+2)})$ when fitted to smooth additive models \cite{tan2021cautionary}. 
To demonstrate the utility of tree-sum models in capturing additive structure, we provide generalization upper bounds of tree-sum models when their structure is chosen by an oracle.
Specifically, we consider the typical supervised learning set-up: $y = f(\bx) + \epsilon$, where $\bx$ is a random variable on $[0,1]^d$ and $\E\braces*{\epsilon~|~\bx} = 0$. 
We assume $f(\bx) = \sum_{k=1}^K f_k(\bx_{I_k})$, where $I_1 \ldots I_K$ are disjoint blocks of features, and $\bx_{I_{k}}$ denotes the sub-vector of $\bx$ comprising coordinates in $I_k$. 
Under this set-up, we show that if each component function $f_k$ is smooth, and blocks of features are independent (i.e., $\bx_{I_j} \indep \bx_{I_k}$ for $j \neq k$), then there exists a tree-sum model such that its squared error generalization upper bound scales as $O(Kn^{-2/(d_{\max}+2)})$ where $d_{max} = \max_{k} |I_k|$. 
It is instructive to consider two extreme cases: If $|I_k|=1$ for each $k$, the upper bound scales as $O\paren*{d n^{-2/3}}$. On the other hand if $K=1$, we have an upper bound of $O\paren*{n^{-2/(d+2)}}$.
Both bounds match the well-known minimax rates for their respective inference problems \cite{raskutti2012minimax}. See  \cref{thm:generalization_main} for the formal statement.  

\paragraph{\methods performs disentanglement.} 
One potential reason for the success of \method~is its ability to disentangle additive structure. 
We show that under the generative model discussed above, if \method~splits nodes using population quantities (i.e., in the large-sample limit), then the set of features split upon in each fitted tree $\hat{f}_k$ is contained within $I_k$.
The precise theorem statement be found in \cref{thm:disentanglement} in Appendix \cref{sec:theory_supp}. 
By disentangling additive components, \method~is able to avoid duplicate subtrees, leading to a more parsimonious model with better performance.
We provide (partial) empirical justification of this in real-world datasets, by showing that FIGS helps to reduce the number of possibly redundant and repeated splits often observed in CART models (see \cref{sec:repeated_splits_supp}.)
As discussed earlier, we emphasize that disentanglement does not necessarily reflect the underlying data generating process.
We refer the reader to \cref{sec:cdr_results} for a larger discussion regarding interpreting disentanglement. 


\section{Bagging-FIGS}
\label{sec:bagging_figs}

Growing deeper decision trees reduces bias but increases variance.  Allowing more splits in a FIGS model has the same effect.
In order to reduce variance, random forests averages the predictions from an ensemble of decision trees that are each grown in a slightly different manner
\cite{breiman2001random}.
Bagging-FIGS averages the predictions from an ensemble of FIGS models, and makes use of the same variance reduction strategies as random forests:
(i) Each FIGS model fit on a bootstrap resampled dataset.
(ii) At each iteration of each FIGS model, a random subset of the original features is chosen, and the algorithm chooses the next split only from this subset of features.
The effect of both these strategies have been studied empirically and theoretically \cite{breiman1996bagging,buhlmann2002analyzing,mentch2020randomization,lejeune2020implicit}.

\paragraph{Baseline methods and settings.}
In the remainder of this section, we will compare the prediction performance of \methods and Bagging-FIGS against that of four other algorithms: random forest, XGBoost, and penalized iteratively reweighted least squares (PIRLS) on the log-likelihood of a generative additive model.
All algorithms are fit using default settings\footnote{We use the implementation of PIRLS in \texttt{pygam} \cite{daniel_serven_2018_1476122}, with 20 splines term for each feature.}.
The number of features subsetted is a tuning parameter for random forest and  Bagging-FIGS.
For both algorithms, we set this to be $d/3$ for regression datasets and $\sqrt{d}$ for classification datasets.
These are the default choices for RF.
We fit Bagging-FIGS with 100 FIGS estimators.


\paragraph{Bagging-FIGS performs comparably to Random Forest and XGBoost}
\cref{fig:generalization_real_world} shows the generalization performance of all the methods on the various real-world datasets introduced in \cref{sec:results} and \cref{sec:cdr_results}.
It shows that the performance of FIGS, measured by AUC can be improved via bagging and feature subsetting (Bagging-FIGS).
In addition, Bagging-FIGS achieves comparable AUC to both XGBoost and random forest (on average improving over XGBoost by 0.015 and random forest by 0.013).
For some datasets (e.g. \textit{IAI pecarn}, \textit{Recidivism}), Bagging-FIGS outperforms both baselines.


\begin{figure*}[h]
     \centering
    \includegraphics[width=\textwidth]{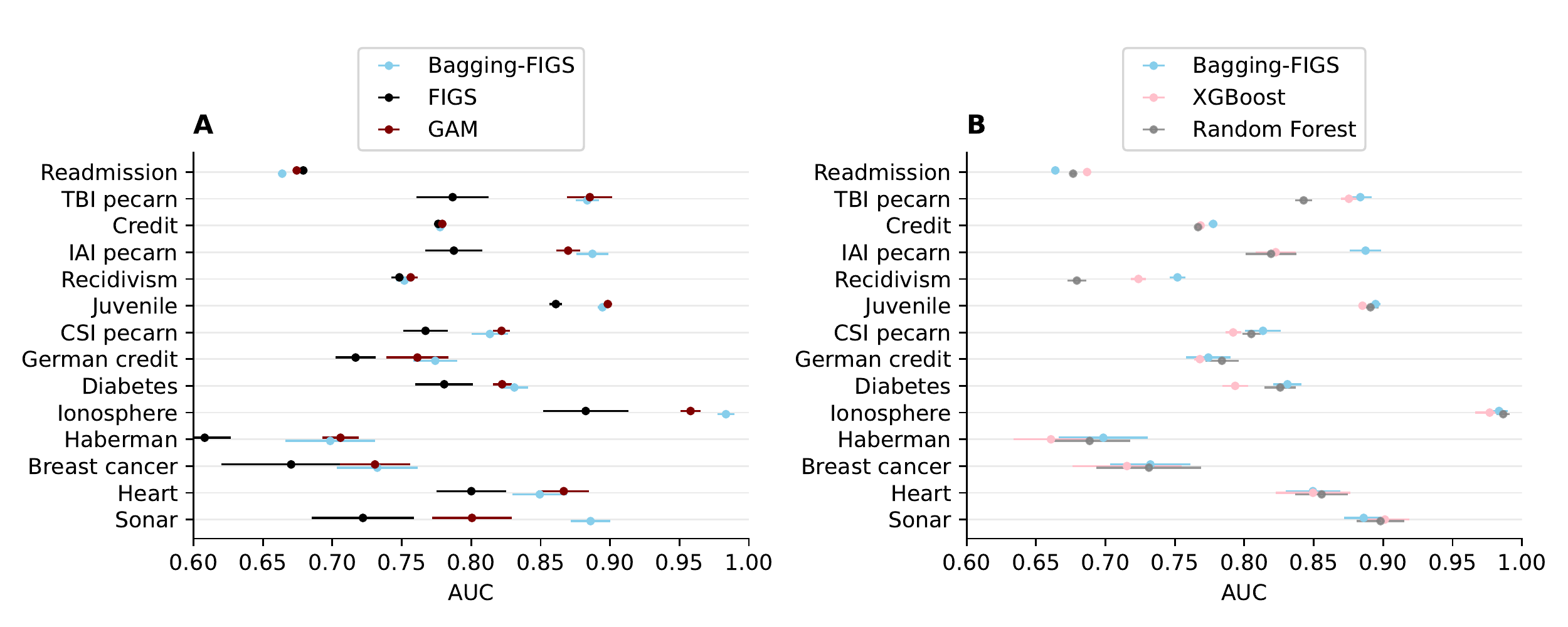}
    \caption{Generalization performance across various real-world datasets with no constraints on the number of splits.
    \textbf{(A)} The performance of \method~can usually be improved by fitting an ensemble of \method~models using bagging and feature subsetting (Bagging-\method).
    The performance of Bagging-\method~is comparable that of a GAM.
    \textbf{(B)} Bagging-\method~achieves comparable performance to both XGBoost and random forest. Error bars show standard error of the mean, computed over 5 random data splits.}
    \label{fig:generalization_real_world}
\end{figure*}

\section{Discussion}
\label{sec:discussion}

\methods is a powerful and natural extension to CART which achieves improved predictive performance over popular baseline tree-based methods across a wide array of datasets while maintaining interpretability by using very few splits.
Furthermore, when the number of splits is unconstrained, an ensemble version of FIGS has prediction performance that compares favorably to random forest and XGBoost.

As a case study, we have shown how
\methods and \methodabbrv~can make an important step towards interpretable modeling of heterogeneous data in the context of high-stakes clinical decision-making.
The fitted CDI models here show promise, but require external clinical validation before potential use.
While our current work only explores age-based grouping, the behavior of \methodabbrv~with temporal, geographical, or demographic splits could be studied as well. Additionally, there are many methodological extensions to explore, such as data-driven identification of input data groups and schemes for feature weighting in addition to instance weighting. \methods has many other natural extensions, some of which we detail below.
\paragraph{Optimization.}
Instead of using a greedy approach, one may perform global optimization algorithm over the class of tree-sum models.
FIGS could include other moves such as pruning in addition to just adding more splits.
This could also be embedded in a Bayesian framework, similar to Bayesian Additive Regression Trees \cite{chipman2010bart}.


\paragraph{Regularization.}

Using cross-validation (CV) to select the number of splits tends to select larger models. Future work can use criteria related to BIC~\cite{schwarz1978estimating} or stability in combination with CV~\cite{lim2016estimation} for selecting this threshold based on data. In future work, one could also vary the total number of splits and number of trees separately, helping to build prior knowledge into the fitting process. \method~could be penalized via novel regularization techniques, such as regularizing individual leaves or regularizing a linear model formed from the splits extracted by \method~\cite{agarwal2022hierarchical}.

\paragraph{Learning interactions.} Interactions are known to be prevalent in biology and other fields, and therefore of key scientific interest \cite{zuk2012mystery}. There has been recent interest in using random forests to learn interactions \cite{basu2018iterative,kumbier2018refining,behr2021provable}. However, since single decision trees are unable to disentangle additive structure, using random forests for interaction discovery might lead to a high false discovery rate. As a result, using \method~in lieu of single decision trees might lead to improved interaction discovery. We leave this investigation to future work.



\paragraph{Model class.}
The class of \methods models could be further extended to include linear terms or allow for summations of trees to be present at split nodes, rather than just at the root.

\vspace{0.1in}

\noindent We hope \methods and G-FIGS can pave the way towards more transparent and interpretable modeling that can improve machine-learning practice, particularly in high-stakes domains such as medicine, law, and policy making.

\FloatBarrier
\section{Acknowledgements}

We gratefully acknowledge partial support from 
NSF TRIPODS Grant 1740855, DMS-1613002, 1953191, 2015341, 2209975 , IIS 1741340, ONR grant N00014-17-1-2176, the Center for Science of Information (CSoI), an NSF Science and Technology Center, under grant
agreement CCF-0939370, NSF grant 2023505 on
Collaborative Research: Foundations of Data Science Institute (FODSI),
the NSF and the Simons Foundation for the Collaboration on the Theoretical Foundations of Deep Learning through awards DMS-2031883 and 814639, and a Weill Neurohub grant. YT was partially supported by NUS Start-up Grant A-8000448-00-00.  AK was supported by the Eunice Kennedy Shriver National Institute of Child Health and Human Development of the National Institutes of Health under Award Number K23HD110716--01.
\FloatBarrier


{
    \section{References}
    \bibliography{_main}
}
\appendix

\counterwithin{figure}{section}
\counterwithin{table}{section}
\renewcommand{\thetable}{S\arabic{table}}
\renewcommand{\thefigure}{S\arabic{figure}}
\renewcommand{\thesection}{S\arabic{section}}

\newpage
\onecolumn


\begin{center}
    \Huge
    \textbf{Supplement}
\end{center}

The supplement contains further details of our investigation into FIGS.
In \cref{sec:run_time_extensions_supp}, 
we discuss computational issues, and further possible extensions of the \method~algorithm. 
In  \cref{sec:sim_results_supp}, we perform a number of synthetic simulations that examine how \method~and Bagging-\method~can adapt to a number of data  generating processes in comparison to other methods. In \cref{sec:repeated_splits_supp}, we provide empirical evidence of disentanglement by showing \method~avoids repeated splits on a number of datasets. In \cref{sec:g_figs_supp}, we provide an extended description of G-\method. \cref{sec:cdi_results_supp} shows the CDIs learnt by G-\method~for the IAI, TBI, and CSI datasets. \cref{sec:cdi_results_supp} also provides extended details on data cleaning for the the IAI, TBI, and CSI datasets, hyper-parameter selection for G-\method, as well as extended results for G-\method. Further, \cref{sec:cdi_results_supp} includes a stability analysis of \methodabbrv~on the CSI dataset. Finally, \cref{sec:theory_supp} includes theoretical investigations into \method~and tree-sum models. 

\section{\methods run-time analysis and extensions}
\label{sec:run_time_extensions_supp}

\textbf{\methods run-time analysis}. \textit{The run time complexity for \method~to grow a model with $m$ splits in total is $O(dm^2n^2)$, where $d$ the number of features, and $n$ the number of samples.}

\begin{proof}
    Each iteration of the outer loop adds exactly one split, so it suffices to bound the running time for each iteration, where it is clear that the cost is dominated by the operation $\texttt{split}$ in~\cref{alg:method} line 9, which takes $O(n^2 d)$. This is because there are at most $n d$ possible splits, and it takes $O(n)$ time to compute the impurity decrease for each of these.
    Consider iteration $\tau$, in which we have a \method~model $f$ with $\tau$ splits. 
    Suppose $f$ comprises $k$ trees in total, with tree $i$ having $\tau_i$ splits, so that $\tau = \tau_1 + \ldots + \tau_k$.
    The total number of potential splits is equal to $l+1$, where $l$ is the total number of leaves in the model.
    The number of leaves for tree $i$ is $\tau_i + 1$, so the total number of leaves in $f$ is
    $$
    l = \sum_{i=1}^k (\tau_i+1) = \tau + k.
    $$
    Since each tree has at least one split, we have $k \leq \tau$, so that the number of potential splits is at most $2\tau+1$
    The total time complexity is therefore
    $$
    \sum_{t=1}^m (2\tau+1) \cdot O(n^2 d) = O(m^2n^2d).
    $$
\end{proof}

\noindent \textbf{\methods extensions: Updating leaf values after tree structures are fixed.} We can continue to update the leaf values of trees in the \methods model after the stopping condition has been reached and the tree structures are fixed.
To do this, we perform backfitting, i.e. we cycle through the trees in the model several times, and at each iteration, update the leaf values of a given tree to minimize the sum of squared residuals of the full model.
This can be seen to be equivalent to block coordinate descent on a linear system, and so converges linearly, under regularity conditions, to the empirical risk minimizer among all functions that can be represented as sum of component functions which are each implementable by one of the tree structures.\footnote{We say that a function is implementable by a tree structure if it is constant on each of the leaves of the tree.}
In our simulations, this postprocessing step usually does not seem to change the leaf values too much, probably because at the moment the leaves are created, their initial values are already chosen to minimize the mean-squared-error.
Hence, we keep the step optional in our implementation of \method.

\section{FIGS Simulation Results.}
\label{sec:sim_results_supp}

We perform simulations to examine how well \methods and Bagging-FIGS adapt to different data generating processes in comparison to four other methods: CART, random forests (RF), XGBoost (XGB), and
generalized additive models (GAMs).



\paragraph{Generative models.}
Let $\bx$ be a random variable with distribution $\pi$ on $[0,1]^d$.
Here, we set $d=50$, let $\epsilon \sim N(0,0.01)$, with $\pi$ uniform on $[0,1]^{50}$ and investigate each of the following four regression functions:\\

    \indent (A) Linear model:\\ 
    \indent \quad $f(\bx) = \sum_{k=1}^{20} x_k + \epsilon$
    
    \indent (B) Single Boolean interaction model:\\ 
    \indent \quad $f(\bx) = \prod_{k=1}^8\indicator\braces*{x_k > 0.1} + \epsilon$
    
    \indent (C) Sum of polynomial interactions model:\\
    \indent\quad $f(\bx) = \sum_{k=1}^5 x_{3k-2}x_{3k-1}x_{3k} + \epsilon$
    
    \indent (D) Local spiky sparse model~\cite{behr2021provable}: \\
    \indent\quad $f(\bx) = \sum_{k=1}^5\indicator\braces*{x_{3k-2}, x_{3k-1}, x_{3k} > 0.5} + \epsilon$

\paragraph{Performance metrics.}
For each choice of regression function, we vary the training set size in a geometrically-spaced grid between 100 and 2500.
For each training set, we fit all five algorithms, and compute their noiseless test MSEs on a common test set of size 500.
The entire experiment is repeated 10 times, and the results averaged across the replicates.

\paragraph{Simulation results.}
\methods and Bagging-\method~predicts well across all four generative models, either performing best or very close to the best amongst all methods compared in moderate sample sizes (\cref{fig:more_sims}). 
All other models suffer from weaknesses:
PIRLS performs best for the linear generative model (A), but performs poorly whenever there are interactions present, i.e., for (B), (C) and (D).
Tree ensembles (XGBoost and RF) perform well when there are interactions but fail to fit the linear generative model.

\begin{figure*}[ht]
    \centering
    {
    \small
    \setlength\tabcolsep{0 pt}
    \begin{tabular}{cccc}
        
         (A) Linear model & (B) Single interaction & (C) Sum of polynomial interactions & (D) Local spiky sparse model \\
         \includegraphics[trim = {0.2cm 0 0.3cm 0},clip,width=0.25\textwidth]{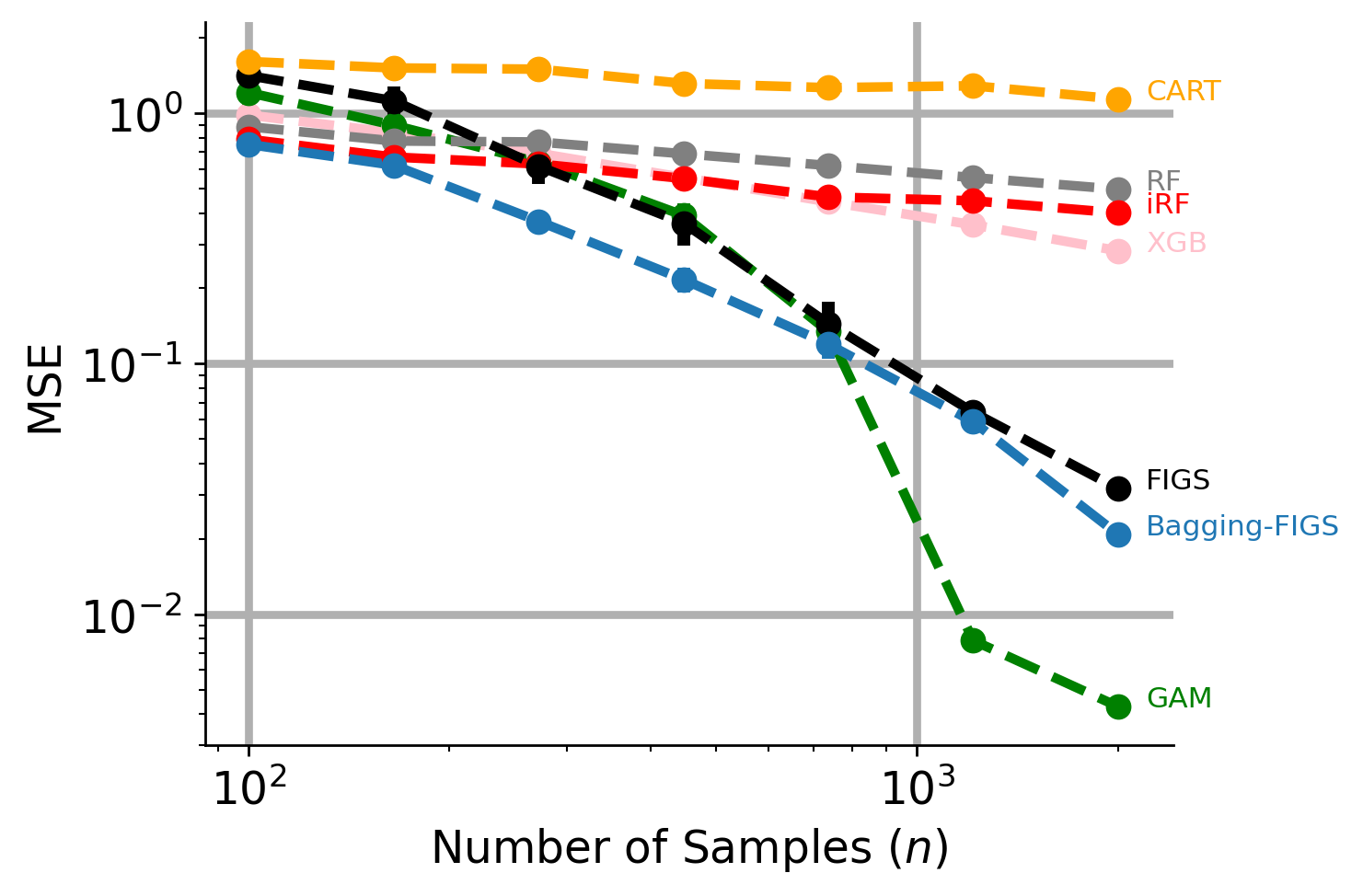} &  \includegraphics[trim = {0.2cm 0 0.4cm 0},clip,width=0.25\textwidth]{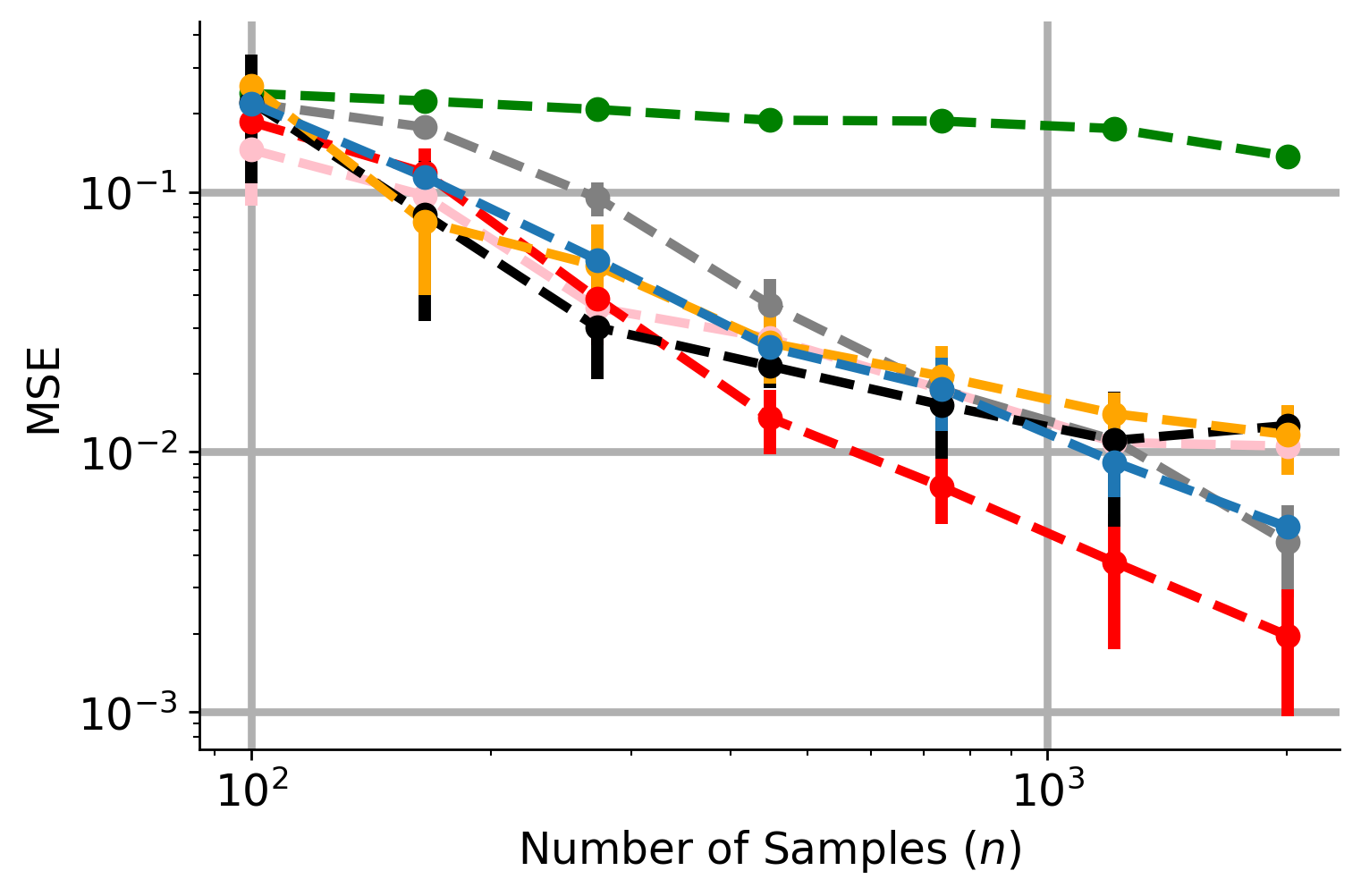} &
         \includegraphics[trim = {0.2cm 0 0.4cm 0},clip,width=0.25\textwidth]{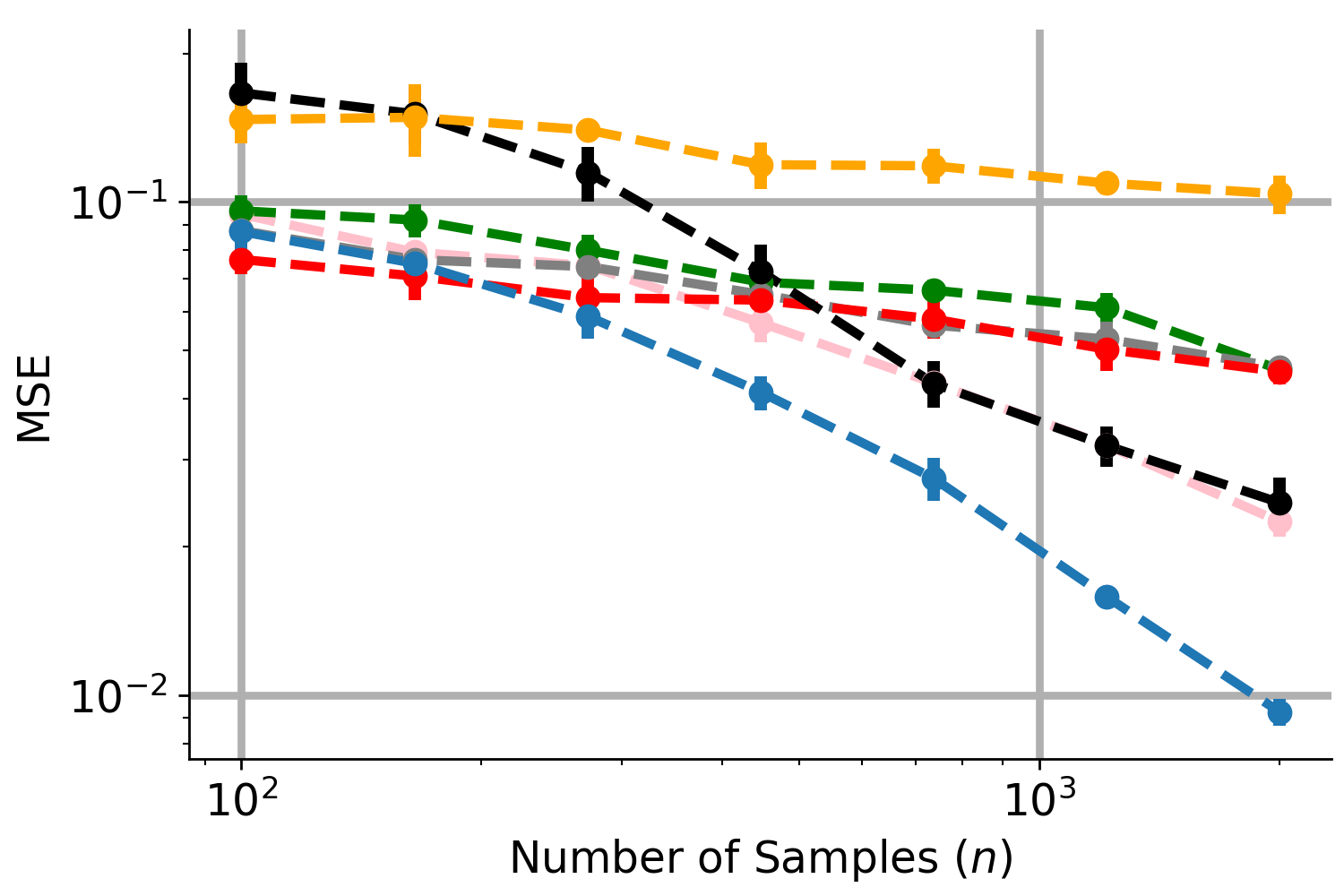} &  \includegraphics[trim = {0.2cm 0 0.4cm 0},clip,width=0.25\textwidth]{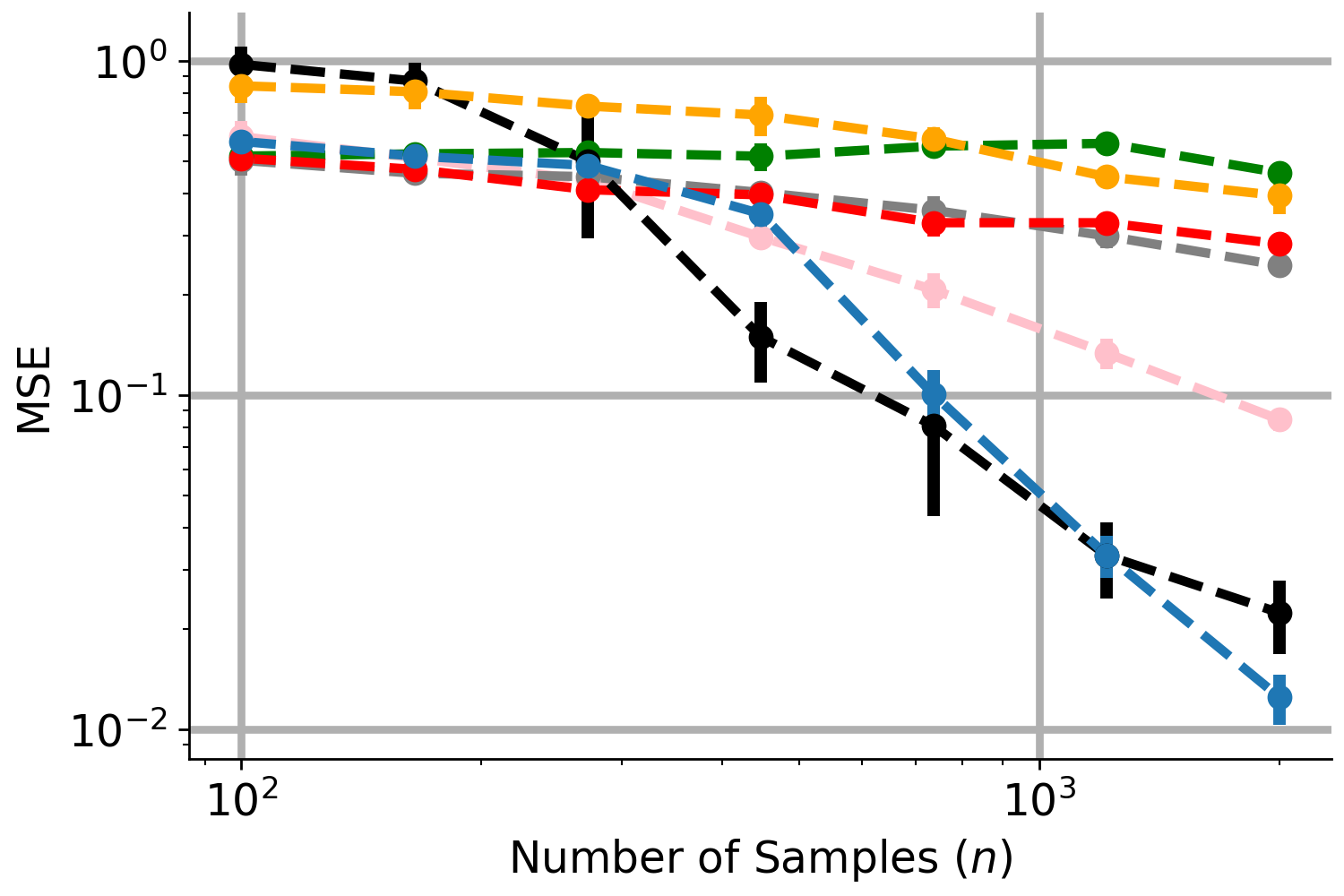}
    \end{tabular}
    \caption{\methods is able to adapt to generative models, handling both additive structure and interactions gracefully.
    In both (A) and (B), the generative model for $X$ is uniform with 50 features.
    Noise is Gaussian with mean zero and standard deviation 0.1 for training but no noise for testing.
    \textit{(A)}
    $Y$ is generated as linear model with sparsity 20 and coefficient 1.
    \textit{(B)} 
    $Y$ is generated from a single Boolean interaction model of order 8.
    \textit{(C)} 
    $Y$ is generated from a sum of 5 three-way polynomial interactions.
    \textit{(D)} 
    $Y$ is generated from a sum of 5 three-way Boolean interactions.
    All results are averaged over 10 runs.
    }
    \label{fig:more_sims}
    }
\end{figure*}

\FloatBarrier
\section{Empirically learned number of trees and repeated splits for FIGS}
\label{sec:repeated_splits_supp}

\cref{fig:repeated_subtrees} investigates whether \methods avoids repeated splits. Specifically, \cref{fig:repeated_subtrees} shows the fraction of splits which are repeated within a learned model as a function of the total number of splits in the model.
We define a split to be repeated if the model contains another split using the same feature and a threshold whose value is within 0.01 of the original split's threshold.\footnote{This result is stable to reasonable variation in the choice of this threshold.}
\methods consistently learns fewer repeated splits than \textcolor{orange}{CART}, one signal that it is avoiding learning redundant subtrees by separately modeling additive components.
\cref{fig:repeated_subtrees} shows the largest three datasets studied.
Finally, \cref{fig:num_trees} shows the number of trees fitted by \methods for each dataset.

\begin{figure}[ht]
    \centering
    \includegraphics[width=0.45\textwidth]{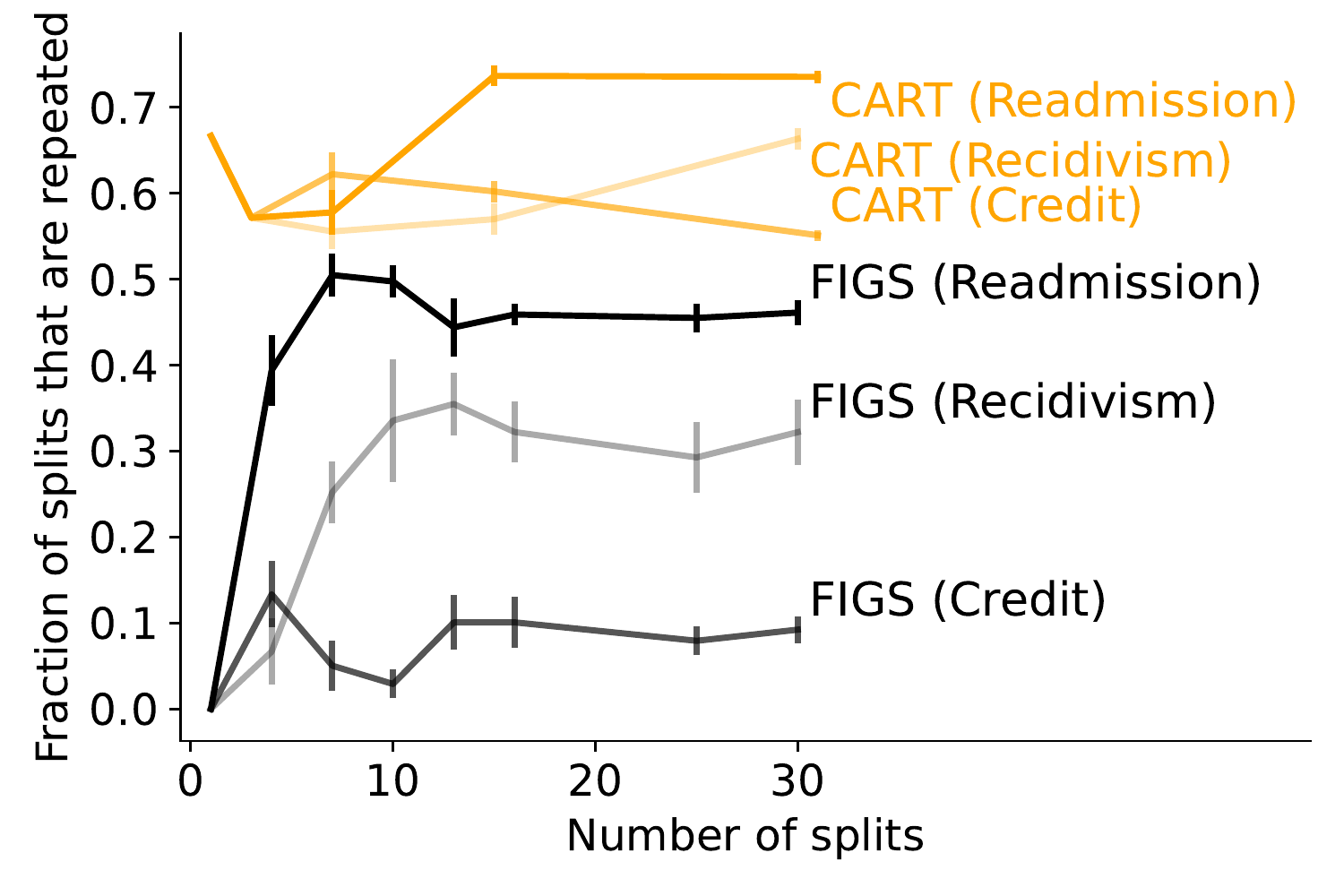}
    \caption{\methods learns less redundant models than CART.
    As a function of the number of splits in the learned model, we plot the fraction of splits
    repeated for the three largest datasets (all datasets studied show the same pattern).
    Error bars show standard error of the mean, computed over 6 random splits.}
    \label{fig:repeated_subtrees}
\end{figure}


\begin{figure}[H]
    \centering
    \includegraphics[width=0.45\columnwidth]{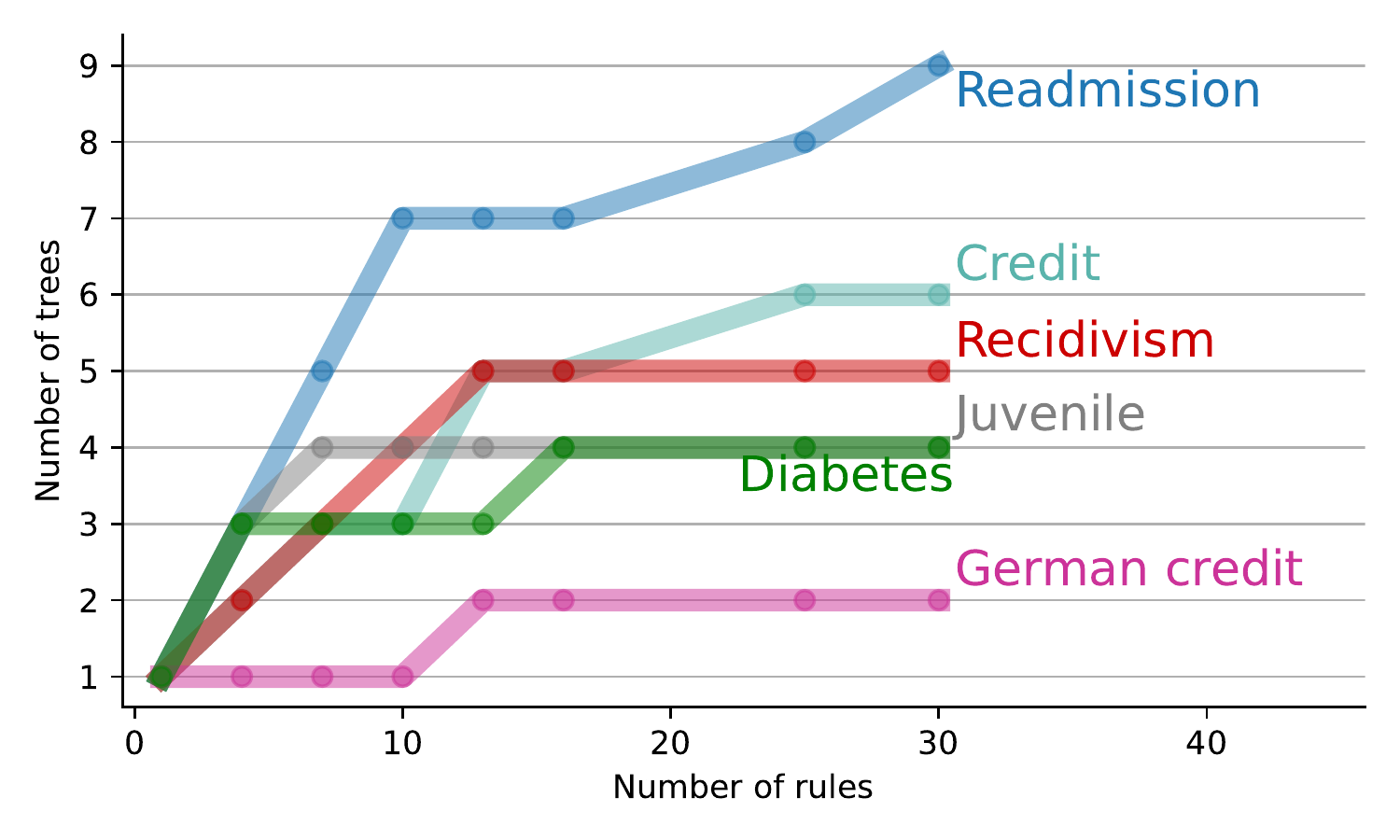}
    \caption{Number of trees learned as a function of the total number of splits in \methods for different classification datasets.}
    \label{fig:num_trees}
\end{figure}

\FloatBarrier
\section{G-FIGS Extended Description}
\label{sec:g_figs_supp}

Group probability-weighted FIGS (\methodabbrv) aims to tackle two challenges: (1) sharing data across heterogenous groups in the input data and (2) ensuring the interpretability of the final output.

\begin{figure}[h]
    \centering
    \includegraphics[width=.7\textwidth]{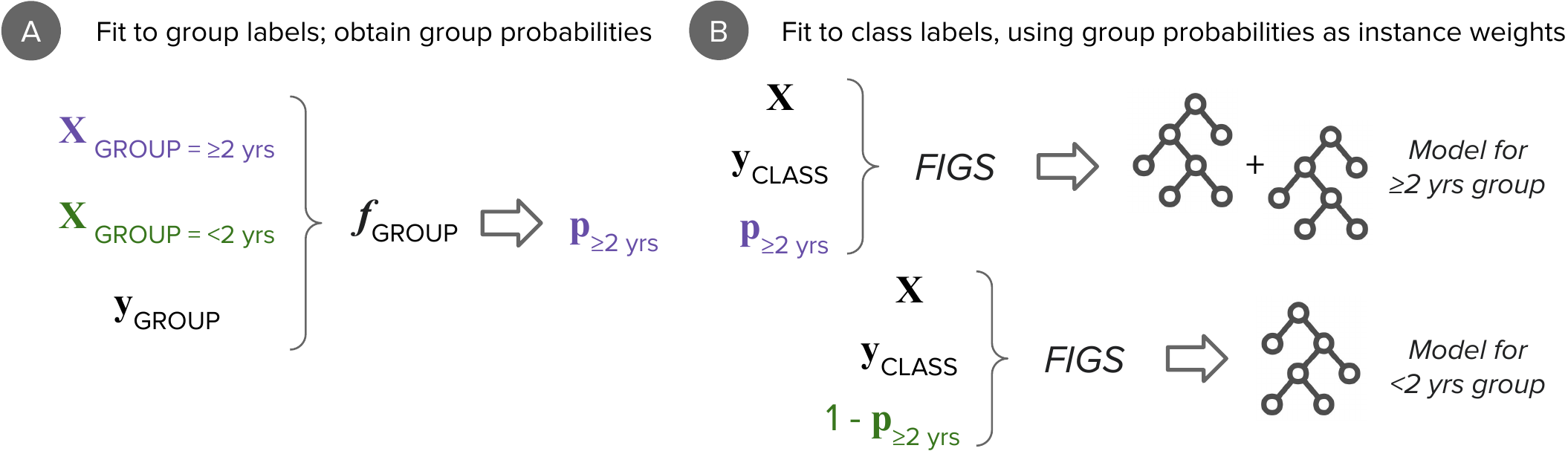}
    \caption{Overview of \methodabbrv. \textbf{(A)} First, the covariates of each instance in a dataset are used to estimate an instance-specific probability of membership in each of the pre-specified groups in the data (e.g., patients of age ${<}2$ years and ${\ge}2$ years). \textbf{(B)} Next, these membership probabilities are used as instance weights when fitting an interpretable model for each group.}
    \label{fig:intro_gfigs}
\end{figure}

\paragraph{Setup.}
We assume a supervised learning setting (classification or regression) with features $X$ (e.g., \textit{blood pressure}, \textit{signs of vomiting}), and an outcome $Y$ (e.g., \textit{cervical spine injury}).
We are also given a group label $G$, which is specified using the context of the problem and domain knowledge; for example, $G$ may correspond to different sites at which data is collected, different demographic groups which are known to require different predictive models, or data before/after a key temporal event.
$G$ should be discrete, as \methodabbrv~will produce a separate model for each unique value of $G$, but may be a discretized continuous or count feature.

\paragraph{Fitting group membership probabilities.}
The first stage of \methodabbrv~fits a classifier to predict group membership probabilities $\P(G|X)$ (\cref{fig:intro_gfigs}A).\footnote{In estimating $\P(G=g|X)$, we exclude features that trivially identify $G$ (e.g., we exclude age when values of $G$ are age ranges).}
Intuitively, these probabilities inform the degree to which a given instance is representative of a particular group; the larger the group membership probability, the more the instances should contribute to the model for that group.
Any classifier can be used; we find that logistic regression and gradient-boosted decision trees perform best.
The group membership probability classifier can be selected using cross-validation, either via group-label classification metrics or downstream performance of the weighted prediction model; we take the latter approach.

\paragraph{Fitting group probability-weighted FIGS.}
In the second stage (\cref{fig:intro_gfigs}B), for each group $G=g$, \methodabbrv~uses the estimated group membership probabilities, $\P(G=g|X)$, as instance weights in the loss function of a ML model for each group $\P(Y|X, G=g)$.
Intuitively, this allows the outcome model for each group to use information from out-of-group instances when their covariates are sufficiently similar.
While the choice of outcome model is flexible, we find that FIGS performs best when both interpretability and high predictive performance are required.\footnote{When interpretability is not critical, the same weighting procedure could also be applied to black-box models, such as Random Forest~\cite{breiman2001random}.} By greedily fitting a sum of trees, FIGS effectively allocates a small budget of splits to different types of structure in data.

\section{CDI Results}
\label{sec:cdi_results_supp}

\subsection{Fitted clinical-decision instruments}
\label{sec:learned_cdis_supp}

The fitted clinical-decision instruments (CDIs) for the IAI, TBI, and CSI datasets from PECARN are shown in \cref{fig:model_ex_iai}, \cref{fig:model_ex_tbi}, and \cref{fig:model_ex_csi_rep} respectively.

\begin{figure}[h!]
    \centering
    \includegraphics[width=0.75\textwidth]{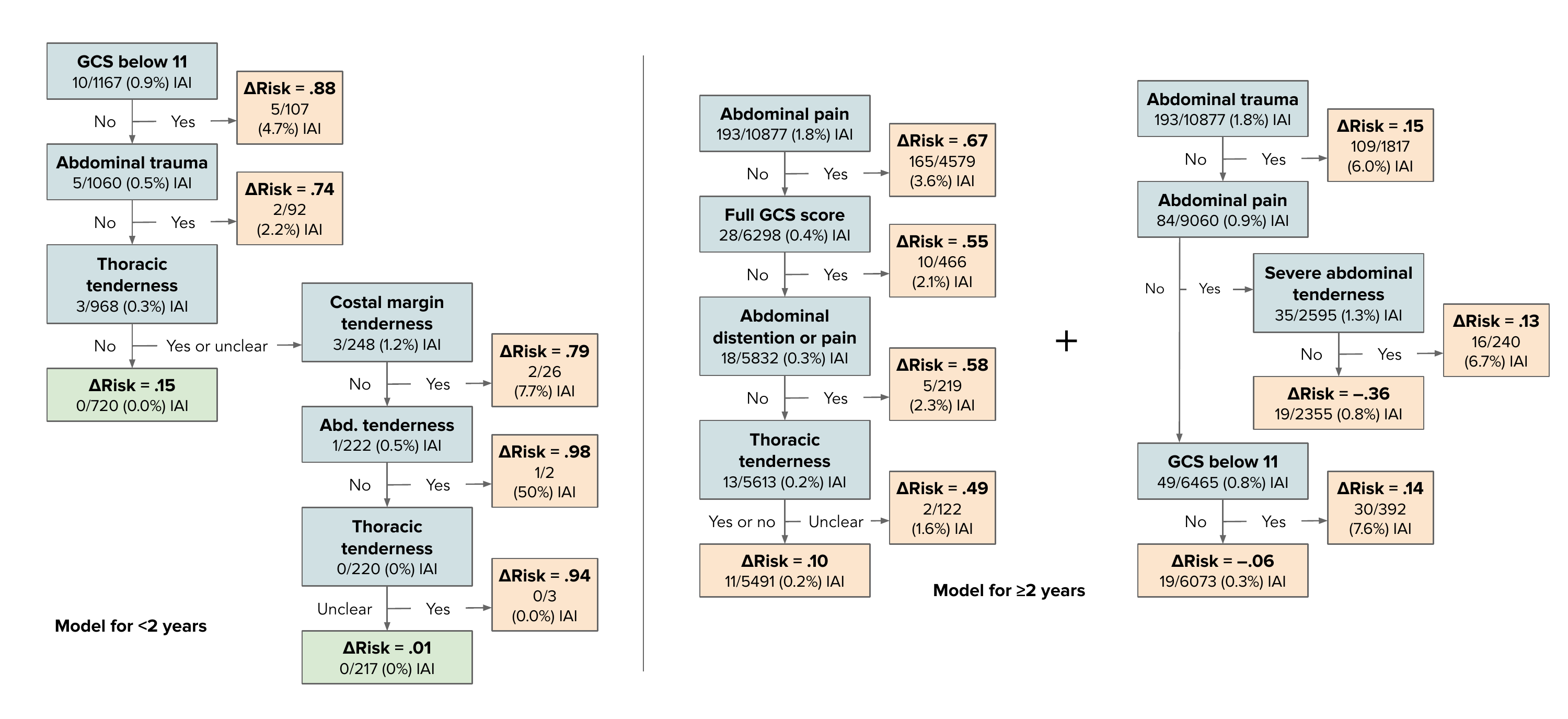}
    \caption{\methodabbrv~model fitted to the IAI dataset. Note that the younger group only uses \textit{tenderness}, which can evaluated without verbal input from the patient, whereas the older group uses \textit{pain}, which requires a verbal response. Achieves 95.1\% sensitivity and 50.8\% specificity (training).}
    \label{fig:model_ex_iai}
\end{figure}

\begin{figure}[h]
    \centering
    \includegraphics[width=0.6\textwidth]{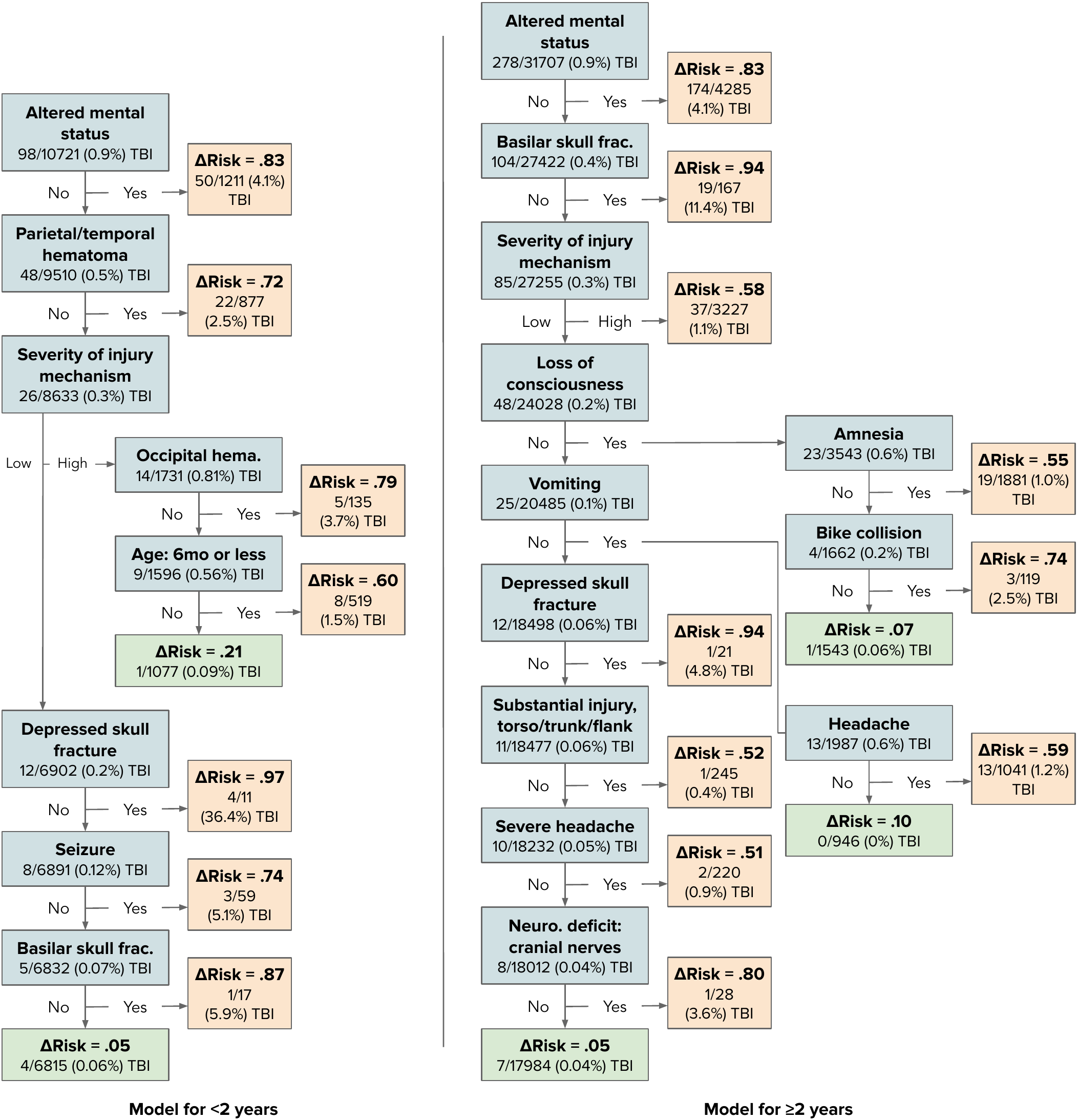}
    \caption{\methodabbrv~model fitted to the TBI dataset. Interestingly, in this case \methodabbrv~learns only a single tree for each group. Note that the model for the older group utilizes the \textit{Headache} and \textit{Severe Headache} features, which require a verbal response. Achieves 97.1\% sensitivity and 58.9\% specificity (training).}
    \label{fig:model_ex_tbi}
\end{figure}

\begin{figure}[H]
    \centering
    \includegraphics[width=0.5\textwidth]{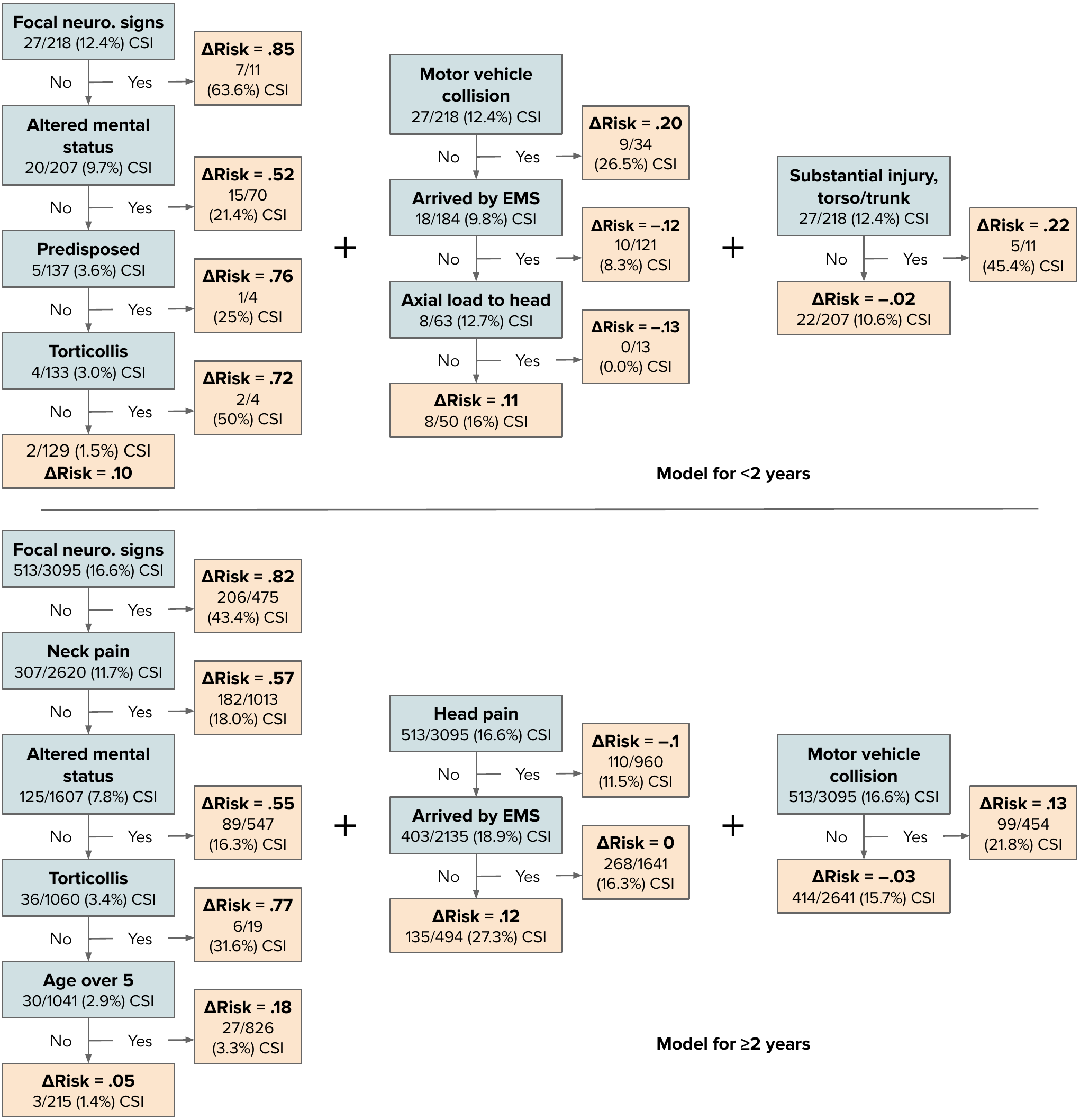}
    \caption{\methodabbrv~model fitted to the CSI dataset. Achieves 97.0\% sensitivity and 33.9\% specificity (training). The left tree for ${<}2$ years gives large $\Delta$ Risk to active features, and on its own provides sensitivity of 99\%.
    Counterintuitively, the middle tree assigns $\Delta$ Risk $<0$ for patients arriving by ambulance (\textit{EMS}) or with head injuries that affect the spine (\textit{axial load}).
    However, adding this second tree results in boosted specificity (increase of 8.7\%) with a tiny reduction in sensitivity (decrease of 0.4\%), indicating that \methodabbrv~adaptively tunes the sensitivity-specificity tradeoff.}
    \label{fig:model_ex_csi_rep}
\end{figure} 

\subsection{Interpreting the group-membership model}
\label{sec:interpreting_group_membership_supp}
Recall that fitting \methodabbrv~requires two steps (see \cref{sec:g_figs_supp}): (1) fitting a group-membership model to obtain sample weights and then (2) using these sample weights to fit a FIGS model.
In this section, we interpret the fitted group-membership models.
In the clinical context, we begin by fitting several logistic regression and gradient-boosted decision tree group membership models to each of the training datasets~(\cref{tab:cdr_datasets}) to predict whether a patient is in the ${<}2$ years or ${\ge}2$ years age group.
For the instance-weighted methods, we treat the choice of group membership model as a hyperparameter, and select the best model according to the downstream performance of the final decision instrument on the validation set. 

\cref{tab:prop_model_variables} shows the coefficients of the most important features for each logistic regression group membership model when predicting whether a patient is in the ${\ge}2$ years age group.
The coefficients reflect existing medical expertise. For example, the presence of verbal response features (e.g., \textit{Amnesia}, \textit{Headache}) increases the probability of being in the ${\ge}2$ years group, as does the presence of
activities not typical for the ${<}2$ years age group (e.g. \textit{Bike injury}).

\begin{table}[htbp]
    \centering
    \footnotesize
    \begin{tabular}{lrlrlr}
\\
    \toprule
    \multicolumn{2}{c}{Traumatic brain injury} 
        & \multicolumn{2}{c}{Cervical spine injury} 
        & \multicolumn{2}{c}{Intra-abdominal injury}\\
     \cmidrule(lr){1-2}
     \cmidrule(lr){3-4}
     \cmidrule(lr){5-6}
     Variable & Coefficient & Variable & Coefficient & Variable & Coefficient\\
     \midrule
     No fontanelle bulging & 3.62 & Neck tenderness & 2.44 & Bike injury & 2.01 \\
     Amnesia & 2.07 & Neck pain & 2.18 & Abdomen pain & 1.66 \\
     Pedestrian struck by vehicle  & 1.44 & Motor vehicle injury: other & 1.54 & Thoracic tenderness & 1.43 \\
     Headache & 1.39 & Hit by car & 1.47 & Hypotension & 1.23 \\
     Bike injury & 1.26 & Substantial injury: extremity & 1.35 & No abdomen pain & 0.98 \\
    \bottomrule
\end{tabular}
    \vspace{-5pt}
    \caption{Logistic regression coefficients for features that contribute to high $\P( \text{age} \ {\ge}2 \ \text{years} \ | \ X)$ reflect  known medical knowledge. For example, features with large coefficients require verbal responses (e.g., \textit{Amnesia}, \textit{Headache}, \textit{Pain}), relate to activities not typical for the ${<}2$ years age group (\textit{Bike injury}), or are specific to older children, e.g., older children should have \textit{No fontanelle bulging}, as cranial soft spots typically close by 2 to 3 months after birth.}
    \label{tab:prop_model_variables}
\end{table}

\subsection{Extended clinical-decision results}
\label{sec:extended_crd_results_supp}

In \cref{tab:cdr_results_young} and \cref{tab:cdr_results_old}, we report the results of G-\methods and other baselines for all data-sets for the  $<2$ and $\geq2$ years age group separately. As seen in the results, we see that G-\methods improve upon the baseline methods and \methods particularly at high levels of sensitivities.
Additionally, in \cref{tab:results_extended}, we include the results from above with their standard errors, as well as additional metrics (AUC and F1 score) for each dataset.

\begin{table}[H]
    \centering
    \footnotesize
        \begin{tabular}{lrrrrrrrrrrrr}
    \toprule
    {} & \multicolumn{4}{c}{Cervical spine injury} &
        \multicolumn{4}{c}{Traumatic Brain Injury} 
        & \multicolumn{4}{c}{Intra-abdominal injury} \\
     \cmidrule(lr){2-5}
     \cmidrule(lr){6-9}
     \cmidrule(lr){10-13}
     Sensitivity level: & 92\% & 94\% & 96\% & 98\% & 92\% & 94\% & 96\% & 98\% & 92\% & 94\% & 96\% & 98\% \\
     \midrule
    TAO & 45.8 & 45.8 & 45.8 & 45.8
        & 7.7 & 7.7 & 0.0 & 0.0 
        & 33.3  & 33.3 & 33.3 & 6.7 \\
     TAO-SEP & 0.0 & 0.0  & 0.0 & 0.0
        & 20.6  & 14.3  & 8.3 & 8.3 
        & 0.0  & 0.0 & 0.0 & 0.0 \\
    CART & 45.8 & 45.8 & 45.8 & 45.8
        & 19.0 & 19.0 & 7.1 & 1.2
        & 29.6 & 29.6 & 29.6 & 29.6 \\
     CART-SEP & 0.0 & 0.0 & 0.0 &  0.0
        & 20.6 & 14.3 & 8.3 & 8.3 
        & 0.0 & 0.0 & 0.0 & 0.0 \\
   G-CART & 36.4 & 36.4 & 36.4 &  36.4
        & 16.1 & 15.6 & 8.7 & \textbf{8.7} 
        & 4.4 & 4.4 & 4.4 & 4.4 \\
     \textbf{FIGS} & 56.1 &  56.1 &  56.1 &  56.1
        & \textbf{36.3} & \textbf{30.3} & 5.9 & 0.1   
        & \textbf{39.4} & \textbf{39.4} &\textbf{39.4}  & \textbf{39.4}   \\
     \textbf{FIGS-SEP} & 6.9 & 6.9 &  6.9 &  6.9 
        & 31.1 & 25.1 & 13.0 & 6.7 
        & 9.0 & 9.0 & 9.0 &  9.0 \\
     \textbf{\methodabbrv} & \textbf{65.9} & \textbf{65.9} & \textbf{65.9} & \textbf{65.9} 
        & 17.2  & 17.2 & \textbf{13.7} & 7.5
        & 24.3 & 24.3 &  24.3 & 24.3 \\
    \bottomrule
    \end{tabular}

        \caption{$<2$ years age group test set prediction results averaged over 10 random data splits. Values in columns labeled with a sensitivity percentage (e.g. 92\%) are best specificity achieved at the given level of sensitivity or greater.
        }
    \label{tab:cdr_results_young}
\end{table}

\begin{table}[H]
    \centering
    \footnotesize
        \begin{tabular}{lrrrrrrrrrrrr}
    \toprule
    {} & \multicolumn{4}{c}{Cervical spine injury} &
        \multicolumn{4}{c}{Traumatic Brain Injury} 
        & \multicolumn{4}{c}{Intra-abdominal injury} \\
     \cmidrule(lr){2-5}
     \cmidrule(lr){6-9}
     \cmidrule(lr){10-13}
     Sensitivity level: & 92\% & 94\% & 96\% & 98\% & 92\% & 94\% & 96\% & 98\% & 92\% & 94\% & 96\% & 98\% \\
     \midrule
    TAO & 39.5 & 20.1 & 0.2 & 0.2
        & 12.2 & 6.1 & 6.1 & 0.3
        & 0.3  & 0.3 & 0.0 & 0.0 \\
     TAO-SEP & 33.6 & 17.9  & 3.5 & 1.4
        & 19.8  & 13.4  & 7.4 & 0.5 
        & 10.2  & 5.5 & 4.2 & 1.3 \\
    CART & 22.0 & 15.4 & 14.8 & 2.1
        & 19.0 & 19.0 & 7.1 & 1.2
        & 7.6 & 2.8 & 1.6 & 1.3 \\
     CART-SEP & 33.0 & 17.3 & 2.8 &  1.4
        & 19.7 & 13.4 & 7.3 &  0.8
        & 9.0 & 4.2 & 4.2 & 1.3 \\
    G-CART & 37.3 & 16.2 & 1.9 &  1.7
        & 19.6 & 7.2 & 7.2 &  0.8
        & 14.7 & 10.5 & 3.9 & 0.7 \\
     \textbf{FIGS} & 37.8 &  33.6 &  \textbf{25.5} &  \textbf{14.1}
        & 24.5 & 18.1 & 18.1 & 0.5   
        & 24.9 & 14.6 & 1.4  & 0.0   \\
     \textbf{FIGS-SEP} & 39.5 & \textbf{33.8} &  22.0 &  11.6 
        & 25.3 & 20.3 & 18.7 & \textbf{6.3} 
        & \textbf{28.0} & 19.0 & 9.2 &  0.5 \\
     \textbf{\methodabbrv} & \textbf{40.7} & 33.5 & 23.8 & 13.5
        & \textbf{44.1}  & \textbf{31.5} & \textbf{19.5} & 5.3
        & 27.9 & \textbf{22.1} &  \textbf{13.8} & \textbf{2.4} \\
    \bottomrule
    \end{tabular}

        \caption{$>2$ years age group test set prediction results averaged over 10 random data splits. Values in columns labeled with a sensitivity percentage (e.g. 92\%) are best specificity achieved at the given level of sensitivity or greater. 
        }
    \label{tab:cdr_results_old}
\end{table}

\begin{table}[H]
    \centering
    \footnotesize
    \begin{tabular}{lrrrrrrrrr}
    \toprule
    {} & \multicolumn{6}{c}{Traumatic brain injury} &
        \multicolumn{3}{c}{Cervical spine injury} \\
     \cmidrule(lr){2-7}
     \cmidrule(lr){8-10}
     {} & 92\% & 94\% & 96\% & 98\% & ROC AUC & F1
     & 92\% & 94\% & 96\% \\
     \midrule
    TAO & 6.2 (5.9) & 6.2 (5.9) & 0.4 (0.4) & 0.4 (0.4) &  .294 (.05) &  5.2 (.00) 
        & 41.5 (0.9) & 21.2 (6.6) & 0.2 (0.2)   \\
     TAO-SEP & 26.7 (6.4) & 13.9 (5.4) & 10.4 (5.5) & 2.4 (1.5) &  .748 (.02) &  \textbf{5.8} (.00)
        & 32.5 (4.9) & 7.0 (1.6) & 5.4 (0.7)  \\
    CART & 20.9 (8.8) & 14.8 (7.6) & 7.8 (5.8) & 2.1 (0.6) &  .702 (.06) &  5.7 (.00)
        & 38.6 (3.6) & 13.7 (5.7) & 1.5 (0.6)  \\
     CART-SEP & 26.6 (6.4) &  13.8 (5.4) &  10.3 (5.5) &  2.4 (1.5) &  .753 (.02) &  5.6 (.00)
        & 32.1 (5.1) & 7.8 (1.5) & 5.4 (0.7)  \\
    G-CART & 15.5 (5.5) &  13.5 (5.7) &   6.4 (2.2) &  3.0 (1.5) &  \textbf{.758} (.01) &  5.5 (.00)
        &  38.5 (3.4) &  15.2 (4.8) & 4.9 (1.0) \\
     FIGS & 23.8 (9.0) &  18.2 (8.5) &  12.1 (7.3) &  0.4 (0.3) &  .380 (.07) &  4.8 (.00)
        & 39.1 (3.0) & 33.8 (2.4) & 24.2 (3.2)   \\
     FIGS-SEP & 39.9 (7.9) & 19.7 (6.8) &  \textbf{17.5} (7.0) &  2.6 (1.6) &  .619 (.05) &  5.1 (.00)
        & 38.7 (1.6) & 33.1 (2.0) & 20.1 (2.6)  \\
     \textbf{\methodabbrv} & \textbf{42.0} (6.6) &  \textbf{23.0} (7.8) & 14.7 (6.5) &  \textbf{6.4} (2.8) &  .696 (.04) &  4.7 (.00)
        & \textbf{42.2} (1.3) & \textbf{36.2} (2.3) & \textbf{28.4} (3.8)  \\
    \midrule
\end{tabular}


\begin{tabular}{lrrrrrrrrr}
    {} & \multicolumn{3}{c}{Cervical spine injury (cont.)}
       & \multicolumn{6}{c}{Intra-abdominal injury}\\
     \cmidrule(lr){2-4}
     \cmidrule(lr){5-10}
     {} & 98\% & ROC AUC & F1 & 92\% & 94\% & 96\% & 98\% & ROC AUC & F1\\
     \midrule
     TAO & 0.2 (0.2) &  .422 (.04) &  44.5 (.01)
        & 0.2 (0.2) & 0.2 (0.2) & 0.0 (0.0) & 0.0 (0.0) & .372 (.04) &  \textbf{13.9} (.01) \\
     TAO-SEP & 2.5 (1.0) &  .702 (.01) &  44.4 (.01)
        & 12.1 (1.7) & 8.5 (2.0) & 2.0 (1.3) & 0.0 (0.0) & .675 (.01) &  12.9 (.00) \\
     CART & 1.1 (0.4) &  .617 (.06) &  \textbf{45.8} (.01)
        & 11.8 (5.0) & 2.7 (1.0) & 1.6 (0.5) & 1.4 (0.5) &  .688 (.06) &  13.4 (.00) \\
     CART-SEP & 2.5 (1.0) &  .707 (.00) &  44.2 (.01)
        & 11.0 (1.6) & 9.3 (1.8) & 2.8 (1.4) & 0.0 (0.0) &  .688 (.01) &  13.0 (.01) \\
     G-CART & 3.9 (1.1)  &  \textbf{.751} (.01) &  45.2 (.01)
        & 11.7 (1.3) &  10.1 (1.6) & 3.8 (1.3) &  0.7 (0.4) &  \textbf{.732} (.02) &  12.5 (.01) \\
     FIGS & \textbf{16.7} (3.9) &  .664 (.03) &  43.0 (.01)
        & \textbf{32.1} (5.5) & 13.7 (6.0) & 1.4 (0.8) & 0.0 (0.0) &  .541 (.04) &   9.4 (.01)\\
     FIGS-SEP & 3.9 (2.2) &  .643 (.02) &  41.4 (.01)
        & 18.8 (4.4) &  9.2 (2.2) & 2.6 (1.7) &  0.9 (0.8) &  .653 (.02) &   8.0 (.00)\\
     \textbf{\methodabbrv} & 15.7 (3.9) &  .700 (.01) &  42.6 (.01)
        & 29.7 (6.9) &  \textbf{18.8} (6.6) & \textbf{11.7} (5.1) & \textbf{3.0} (1.3) &  .671 (.03) &   9.1 (.01) \\
    \bottomrule
\end{tabular}
        \caption{Test set prediction results averaged over 10 random data splits, with corresponding standard error in parentheses. Values in columns labeled with a sensitivity percentage (e.g. 92\%) are best specificity achieved at the given level of sensitivity or greater. \methodabbrv~provides the best performance overall in the high-sensitivity regime. G-CART attains the best ROC curves, while TAO is strongest in terms of F1 score. 
        }
    \label{tab:results_extended}
\end{table}

\subsection{Clinical data-preprocessing details}
\label{sec:data_preprocessing_supp}

\paragraph{Traumatic brain injury (TBI)}

To screen patients, we follow the inclusion and exclusion criteria from previous work \cite{kuppermann2009identification}, which excludes patients with Glasgow Coma Scale (GCS) scores under 14 or no signs or symptoms of head trauma, among other disqualifying factors.
No patients were dropped due to missing values: the majority of patients have about 1\% of features missing, and are at maximum still under 20\%. We utilize the same set of features as a previous study~\cite{kuppermann2009identification}.

Our strategy for imputing missing values differed between features according to clinical guidance. For features that are unlikely to be left unrecorded if present, such as paralysis, missing values were assumed to be negative. For other features that could be unnoticed by clinicians or guardians, such as loss of consciousness, missing values are assumed to be positive. For features that did not fit into either of these groups or were numeric, missing values are imputed with the median.

\paragraph{Cervical spine injury (CSI)}

\cite{leonard2019cervical} engineered a set of 22 expert features from 609 raw features; we utilize this set but add back features that provide information on the following:
\begin{itemize}
    \vspace{-5pt}
    \setlength\itemsep{0.1em}
    \item Patient position after injury
    \item Clinical intervention received by patients prior to arrival (immobilization, intubation)
    \item Pain and tenderness of the head, face, torso/trunk, and extremities
    \item Age and gender
    \item Whether the patient arrived by emergency medical service (EMS)
    \vspace{-5pt}
\end{itemize}

We follow the same imputation strategy described in the TBI paragraph above. Features that are assumed to be negative if missing include focal neurological findings, motor vehicle collision, and torticollis, while the only feature assumed to be positive if missing is loss of consciousness.

\paragraph{Intra-abdominal injury (IAI) }

We follow the data preprocessing steps described in \cite{holmes2013identification} and \cite{kornblith2022predictability}. In particular, all features of which at least 5\% of values are missing are removed, and variables that exhibit insufficient inter-rater agreement (lower bound of 95\% CI under 0.4) are removed.
The remaining missing values are imputed with the median.
In addition to the 18 original variables, we engineered three additional features:
\begin{itemize}
    \vspace{-5pt}
    \setlength\itemsep{0.1em}
    \item \textit{Full GCS score}: True when GCS is equal to     the maximum score of 15
    \item \textit{Abd. Distention or abd. pain}: Either         abdominal distention or abdominal pain 
    \item \textit{Abd. trauma or seatbelt sign}: Either         abdominal trauma or seatbelt sign
    \vspace{-5pt}
\end{itemize}

\paragraph{Data for predicting group membership probabilities}

The data preprocessing steps for the group membership models in the first step of \methodabbrv~are identical to that above, except that missing values are not imputed at all for categorical features, such that ``missing", or NaN, is allowed as one of the feature labels in the data. We find that this results in more accurate group membership probabilities, since for some features, such as those requiring a verbal response, missing values are predictive of age group.

Unprocessed data is available at \url{https://pecarn.org/datasets/} and clean data is available on github at 
\url{https://github.com/csinva/imodels-data} (easily accessibly through the imodels package~\cite{singh2021imodels}).

\begin{table}[H]
    \centering
    \footnotesize
    
\begin{tabular}{cc}

\begin{tabular}{lrr}
\toprule
\multicolumn{3}{c}{Traumatic brain injury} \\
\midrule
                Feature Name &  \% Missing & \% Nonzero \\
\midrule
    Altered Mental Status &              0.74 &     12.95 \\
    Altered Mental Status: Agitated &             87.05 &      1.79 \\
    Altered Mental Status: Other &             87.05 &      1.82 \\
    Altered Mental Status: Repetitive &             87.05 &      1.04 \\
    Altered Mental Status: Sleepy &             87.05 &      6.67 \\
    Altered Mental Status: Slow to respond &             87.05 &      3.22 \\
    Acting normally per parents &              7.09 &     85.38 \\
    Age (months) &              0.00 &       N/A \\
    Verbal amnesia &             38.41 &     10.45 \\
    Trauma above clavicles &              0.30 &     64.38 \\
    Trauma above clavicles: Face &             35.92 &     29.99 \\
    Trauma above clavicles: Scalp-frontal &             35.92 &     20.48 \\
    Trauma above clavicles: Neck &             35.92 &      1.38 \\
    Trauma above clavicles: Scalp-occipital &             35.92 &      9.62 \\
    Trauma above clavicles: Scalp-parietal &             35.92 &      7.79 \\
    Trauma above clavicles: Scalp-temporal &             35.92 &      3.39 \\
    Drugs suspected &              4.19 &      0.87 \\
    Fontanelle bulging &              0.37 &      0.06 \\
    Sex &              0.01 &       N/A \\
    Headache severity &              2.38 &       N/A \\
    Headache start time &              3.09 &       N/A \\
    Headache &             32.76 &     29.94 \\
    Hematoma &              0.69 &     39.42 \\
    Hematoma location &              0.47 &       N/A \\
    Hematoma size &              1.67 &       N/A \\
    Severity of injury mechanism &              0.74 &       N/A \\
    Injury mechanism &              0.67 &       N/A \\
    Intubated &              0.73 &      0.01 \\
    Loss of consciousness &              4.05 &     10.37 \\
    Length of loss of consciousness&              5.39 &       N/A \\
    Neurological deficit &              0.85 &       1.3 \\
    Neurological deficit: Cranial &             98.70 &      0.18 \\
    Neurological deficit: Motor &             98.70 &      0.28 \\
    Neurological deficit: Other &             98.70 &      0.71 \\
    Neurological deficit: Reflex &             98.70 &      0.03 \\
    Neurological deficit: Sensory &             98.70 &      0.26 \\
    Other substantial injury  &              0.43 &     10.07 \\
    Other substantial injury: Abdomen &             89.93 &      1.25 \\
    Other substantial injury: Cervical spine &             89.93 &      1.37 \\
    Other substantial injury: Cut &             89.93 &      0.12 \\
    Other substantial injury: Extremity &             89.93 &      5.49 \\
    Other substantial injury: Flank &             89.93 &      1.56 \\
    Other substantial injury: Other &             89.93 &      1.65 \\
    Other substantial injury: Pelvis &             89.93 &      0.44 \\
    Paralyzed &              0.75 &      0.01 \\
    Basilar skull fracture &              0.99 &      0.68 \\
    Basilar skull fracture: Hemotympanum &             99.32 &      0.35 \\
    Basilar skull fracture: CSF otorrhea &             99.32 &      0.04 \\
    Basilar skull fracture: Periorbital &             99.32 &      0.19 \\
    Basilar skull fracture: Retroauricular &             99.32 &      0.08 \\
    Basilar skull fracture: CSF rhinorrhea & 99.32 &      0.03 \\
    Skull fracture: Palpable &              0.24 &      0.38 \\
    Skull fracture: Palpable and depressed &             99.69 &      0.18 \\
    Sedated &              0.76 &      0.08 \\
    Seizure &              1.70 &      1.17 \\
    Length of seizure &              0.18 &       N/A \\
    Time of seizure &              0.12 &       N/A \\
    Vomiting &              0.71 &      13.1 \\
    Time of last vomit &             89.04 &       N/A \\
\end{tabular}
 &

\begin{tabular}{lrr}

Number of times vomited &              0.60 &       N/A \\
Vomit start time &              0.87 &       N/A \\
\midrule
\multicolumn{3}{c}{Intra-abdominal injury} \\
\midrule
    Abdominal distention &              4.38 &       2.3 \\
    Abdominal distention or pain &              0.00 &      4.93 \\
    Degree of abdominal tenderness &             70.13 &       N/A \\
    Abdominal trauma &              0.56 &     15.48 \\
    Abdominal trauma or seat belt sign & 0.00 &      16.3 \\
    Abdomen pain &             15.38 &     30.06 \\
    Age (years) &              0.00 &       N/A \\
    Costal margin tenderness &              0.00 &     11.33 \\
    Decreased breath sound &              1.93 &      2.13 \\
    Distracting pain &              7.38 &     23.29 \\
    Glasgow Coma Scale (GCS) score &              0.00 &       N/A \\
    Full GCS score &              0.00 &     86.21 \\
    Hypotension &              0.00 &      1.44 \\
    Left costal margin tenderness &              0.00 &       N/A \\
    Method of injury &              3.95 &       N/A \\
    Right costal margin tenderness &              0.00 &       N/A \\
    Seat belt sign &              3.30 &      4.93 \\
    Sex &              0.00 &       N/A \\
    Thoracic tenderness &              9.99 &     15.96 \\
    Thoracic trauma &              0.63 &     16.95 \\
    Vomiting &              3.92 &      9.57 \\
\midrule
\multicolumn{3}{c}{Cervical spine injury} \\
\midrule
                  Age (years) &              0.00 &       N/A \\
        Altered mental status &              2.05 &     24.72 \\
             Axial load to head &              0.00 &      24.0 \\
               Clotheslining &              3.38 &      0.94 \\
         Focal neurological findings &              9.84 &     14.67 \\
              Method of injury: Diving &              0.03 &       1.3 \\
                Method of injury: Fall &              2.44 &      3.83 \\
             Method of injury: Hanging &              0.03 &      0.15 \\
            Method of injury: Hit by car &              0.03 &     15.09 \\
                 Method of injury: Auto collision &              7.73 &     14.73 \\
             Method of injury: Other auto &              0.03 &      3.11 \\
                Arrived by EMS &              0.00 &     77.24 \\
                         Loss of consciousness  &              8.03 &     42.68 \\
                   Neck pain &              5.25 &     38.42 \\
       Posterior midline neck tenderness &              2.57 &     29.88 \\
                    Patient position on arrival &             61.52 &       N/A \\
                 Predisposed &              0.00 &      0.66 \\
              Pain: Extremity &             18.35 &     25.87 \\
             Pain: Face &             18.35 &      7.58 \\
             Pain: Head &             18.35 &     29.04 \\
       Pain: Torso/trunk &             18.35 &     28.95 \\
                Tenderness: Extremity &             20.37 &     15.15 \\
               Tenderness: Face &             20.37 &      3.83 \\
               Tenderness: Head &             20.37 &      7.79 \\
         Tenderness: Torso/trunk &             20.37 &     25.87 \\
                  Substantial injury: Extremity &              1.03 &     10.87 \\
                 Substantial injury: Face &              1.06 &      5.67 \\
                 Substantial injury: Head &              1.00 &     15.88 \\
           Substantial injury: Torso/trunk &              1.03 &       7.3 \\
                 Neck tenderness &              2.48 &      39.3 \\
                Torticollis &              7.03 &      5.77 \\
                  Ambulatory &              5.77 &     21.46 \\
                Axial load to top of head &              0.00 &      2.35 \\
                Sex &              0.00 &       N/A \\
\bottomrule
\end{tabular}
\end{tabular}

        \caption{Final features used for fitting the \textit{outcome} models. Features include information about patient history (i.e. \textit{mechanism of injury}), physical examination (i.e. \textit{Abdominal trauma}), and mental condition (i.e. \textit{Altered mental status}). Percentage of nonzero values is marked \textit{N/A} for non-binary features.}
    \label{tab:features}
\end{table}

\subsection{Clinical-data hyperparameter selection}
\label{sec:hyperparams_supp}

\paragraph{Data splitting} We use 10 random training/validation/test splits for each dataset, performing hyperparameter selection separately on each. There are two reasons we choose not to use a fixed test set. First, the small number of positive instances in our datasets makes our primary metrics (specificity at high sensitivity levels) noisy, so averaging across multiple splits makes the results more stable. Second, the works that introduced the TBI, IAI, and CSI datasets did not publish their test sets, as it is not as common to do so in the medical field as it is in machine learning, making the choice of test set unclear. For TBI and CSI, we simply use the random seeds 0 through 10. For IAI, some filtering of seeds is required due to the low number of positive examples; we reject seeds that do not allocate positive examples evenly enough between each split (a ratio of negative to positive outcomes over 200 in any split).

\paragraph{Class weights} Due to the importance of achieving high sensitivity, we upweight positive instances in the loss by the inverse proportion of positive instances in the dataset. This results in class weights of about 7:1 for CSI, 112:1 for TBI, and 60:1 for IAI. These weights are fixed for all methods.

\paragraph{Hyperparameter settings} Due to the relatively small number of positive examples in all datasets, we keep the hyperparameter search space small to avoid overfitting. We vary the maximum number of tree splits from 8 to 16 for all methods and the maximum number of update iterations from 1 to 5 for TAO. The options of group membership model are logistic regression with L2 regularization and gradient-boosted trees \citep{friedman2001greedy}. For both models, we simply include two hyperparameter settings: a less-regularized version and a more-regularized version, by varying the inverse regularization strength ($C$) for logistic regression and the number of trees ($N$) for gradient-boosted trees. We initially experimented with random forests and CART, but found them to lead to poor downstream performance. Random forests tended to separate the groups too well in terms of estimated probabilities, leading to little information sharing between groups, while CART did not provide unique enough membership probabilities, since CART probability estimates are simply within-node class proportions.

\paragraph{Validation metrics} We use the highest specificity achieved when sensitivity is at or above 94\% as the metric for validation. If this metric is tied between different hyperparameter settings of the same model, specificity at 90\% sensitivity is used as the tiebreaker. For the IAI dataset, only specificity at 90\% sensitivity is used, since the relatively small number of positive examples makes high sensitivity metrics noisier than usual. If there is still a tie at 90\% sensitivity, the smaller model in terms of number of tree splits is chosen.

\paragraph{Validation of group membership model} Hyperparameter selection for \methodabbrv~and G-CART is done in two stages due to the need to select the best group membership model. First, the best-performing maximum of tree splits is selected for each combination of method and membership model. This is done separately for each data group. Next, the best membership model is selected using the overall performance of the best models across both data groups. The two-stage validation process ensures that the ${<}2$ years and ${\ge}2$ years age groups use the same group membership probabilities, which we have found performs better than allowing different sub-models of \methodabbrv~to use different membership models.

\begin{table}[H]
    \centering
    \footnotesize
    \begin{tabular}{lrrrrrr}
\toprule
{} & \multicolumn{3}{c}{${<}2$ years group} &
        \multicolumn{3}{c}{${\ge}2$ years group} \\
    \cmidrule(lr){2-4}
    \cmidrule(lr){5-7}
Maximum tree splits: &           8 &          12 &          16 &           8 &          12 &          16 \\
\midrule
TAO (1 iter)                 &  \textbf{15.1} (6.7) &  15.1 (6.7) &  14.4 (6.1) &  \textbf{14.1} (7.8) &  14.1 (7.8) &   8.9 (5.9) \\
TAO (5 iter)                &  \textbf{14.4} (6.1) &   0.0 (0.0) &   0.0 (0.0) &   \textbf{8.9} (5.9) &   3.1 (0.9) &   1.5 (0.7) \\
CART-SEP                    &  \textbf{15.1} (6.7) &  14.4 (6.1) &   0.0 (0.0) &  \textbf{14.0} (7.8) &   8.9 (5.9) &   3.1 (0.9) \\
FIGS-SEP                    &  \textbf{13.7} (5.9) &   0.0 (0.0) &   0.0 (0.0) &  \textbf{23.1} (8.8) &  13.0 (7.4) &   7.8 (5.6) \\
G-CART w/ LR ($C = 2.8$)       &   \textbf{7.9} (6.7) &   3.1 (2.1) &   3.5 (1.7) &  19.0 (8.8) &  \textbf{21.8} (8.4) &   2.1 (0.6) \\
G-CART w/ LR ($C = 0.1$)       &  \textbf{20.4} (8.6) &   8.3 (6.6) &  10.1 (6.7) &  12.7 (7.6) &  \textbf{14.9} (7.1) &   3.6 (0.9) \\
G-CART w/ GB ($N = 100$)       &  \textbf{19.8} (8.3) &   7.2 (6.3) &   7.6 (6.1) &  13.3 (8.0) &  \textbf{21.4} (8.5) &   9.0 (5.6) \\
G-CART w/ GB ($N = 50$)        &  \textbf{26.8} (9.7) &   8.1 (6.3) &   8.4 (6.1) &  13.3 (8.0) &  \textbf{21.4} (8.5) &   9.7 (5.6) \\
\methodabbrv~w/ LR ($C = 2.8$) &  \textbf{14.9} (8.5) &   7.5 (5.4) &   8.1 (6.9) &  41.0 (8.7) &  \textbf{48.1} (8.2) &  35.6 (8.9) \\
\methodabbrv~w/ LR ($C = 0.1$) &  \textbf{31.0} (9.4) &  23.1 (9.1) &  25.9 (9.7) &  46.9 (8.4) &  \textbf{48.2} (8.4) &  33.7 (8.9) \\
\methodabbrv~w/ GB ($N = 100$) &  \textbf{24.5} (8.6) &  24.0 (9.3) &  21.2 (8.7) &  \textbf{47.5} (8.5) &  47.5 (8.2) &  27.9 (8.6) \\
\methodabbrv~w/ GB ($N = 50$)  &  \textbf{32.1} (9.6) &  18.3 (8.2) &  12.7 (6.9) &  47.5 (8.5) &  \textbf{53.2} (7.3) &  28.4 (8.3) \\
\bottomrule
\vspace{1pt}
\end{tabular}

(a)

\vspace{10pt}

\begin{tabular}{lrrrr}
\toprule
Group membership model:  & LR ($C = 2.8$) & LR ($C = 0.1$) & GB ($N = 100$) & GB ($N = 50$) \\
\midrule
G-CART (${<}2$ years, ${\ge}2$ years models combined) &   \textbf{27.8} (6.0) &   21.5 (5.9) &     19.0 (5.7) &    27.1 (6.5) \\
\methodabbrv~(${<}2$ years, ${\ge}2$ years models combined) &   51.3 (5.8) &   54.5 (6.2) &     \textbf{57.4} (5.6) &    44.6 (7.4) \\
\bottomrule
\vspace{1pt}
\end{tabular}

(b)
        \caption{Hyperparameter selection table for the TBI dataset; the metric shown is specificity at 94\% sensitivity on the validation set, with corresponding standard error in parentheses. First, the best-performing maximum of tree splits is selected for each method or combination of method and membership model (a). This is done separately for each data group. Next, the best membership model is selected for G-CART and \methodabbrv~using the overall performance of the best models from (a) across both data groups (b). The two-stage validation process ensures that the ${<}2$ years and ${\ge}2$ years age groups use the same group membership probabilities, which we have found leads to better performance than allowing them to use different membership models. Metrics shown are averages across the 10 validation sets, but hyperparameter selection was done independently for each of the 10 data splits.}
    \label{tab:cv}
\end{table}

\subsection{PCS analysis of CSI G-FIGS model} In this section, we perform a stability analysis of the \methodabbrv~fitted on the CSI dataset. As discussed earlier, stability is a crucial pre-requisite for using ML models to interpret real-world scientific problems.  To measure the stability of \methodabbrv, we artificially introduce noise by randomly swapping a percentage $p$ of labels $\by$. We vary $p$ between $\{1\%,2.5\%,5\%\}$. For each value of $p$, we measure stability by comparing the similarity of the feature sets selected in the model trained on the perturbed data to the model displayed in \cref{fig:model_ex_csi}. The similarity of feature sets is measured via the Jaccard distance. That is, let $\hat{f}$ denote the \method~model fitted on the unperturbed data. Further, define $\hat{S}$ as the features split on in $\hat{f}$. Similarly, let $\hat{f}^p$ denote the \method~fitted on the perturbed data. Note that $\hat{f}^p$ does not necessarily need to consist of the same number of trees as $\hat{f}$. Moreover, we define $\hat{S}^p$ as the features split on in $\hat{f}^p$. Then, we define the stability score of $\hat{f}$ as follows
\begin{equation}
\label{eq:stability_score}
    \text{Sta}(\hat{f}) = \frac{\hat{S} \cap \hat{S^p}}{\hat{S} \cup \hat{S^p}}
\end{equation}
For \methodabbrv~we compute the stability score of the model fitted on each age group. We measure the stability score of \methodabbrv~over 5 repetitions, and display the results in \cref{tab:pcs_analysis}. 

\begin{table} [H]
\centering
\footnotesize
\begin{tabular}{lrrr}
\toprule
percentage of labels swapped $p$:  & $1\%$ & $2.5\%$ & $5\%$  \\
\midrule
Sta(\methodabbrv~(${<}2$ years)) &   1.0  &   0.98  &  0.93   \\
Sta(\methodabbrv~(${>}2$ years)) &  0.86     &  0.64   & 0.64  \\
\bottomrule
\vspace{1pt}
\end{tabular}
\caption{Stability analysis results for \methodabbrv~fitted on the CSI dataset when a percentage $p$ of labels are flipped. We vary $p$ betweem $\{1\%,2.5\%,5\%\}$. The stability of \methodabbrv~is measured via \eqref{eq:stability_score}, and we average our results over over 5 repetitions. The results indicate \methodabbrv~is stable to perturbations, in particular for the ${<}2$ years age group.  }
\label{tab:pcs_analysis}
\end{table}



\section{Theoretical Results}
\label{sec:theory_supp}
In this section, we discuss our theoretical results relating to \method~and tree-sum models. 

\subsection{CART as local orthogonal greedy procedure}
\label{sec:supp_cart_orthogonal_supp}

In this section, we build on recent work which shows that CART can be thought of as a ``local orthogonal greedy procedure''~\cite{klusowski2021universal}.
To see this, consider a tree model $\hat f$, and a leaf node $\node$ in the tree.
Given a potential split $s$ of $\node$ into children $\node_L$ and $\node_R$, we may associate the normalized decision stump
\begin{equation} \label{eq:decision_stump}
\psi_{\node,s}(\bx) = \frac{N(\mathfrak{t}_{R})\mathbf{1}\{\bx \in \mathfrak{t}_{L}\} - N(\mathfrak{t}_{L})\{\bx \in \mathfrak{t}_{R}\}}{\sqrt{N(\node)N(\mathfrak{t}_{L})N(\mathfrak{t}_{R})}},
\end{equation}
where $N(-)$ is used to denote the number of samples in a given node.
We use $\bPsi_{\node,s}$ to denote the vector in $\R^n$ comprising its values on the training set, noticing that it has unit norm.
If $\node$ is an interior node, then there is already a designated split $s(\node)$, and we drop the second part of the subscript.
It is easy to see that the collection $\braces*{\bPsi_\node}_{\node \in \hat f}$ is orthogonal to each other, and also to all decision stumps associated to potential splits.
This gives the second equality in the following chain
\begin{equation} \label{eq:impurity_dec_as_squared_dot_prod}
    \hat \Delta(s,\node) = \paren{\by^T \bPsi_{\node,s}}^2 = \paren{\br^T \bPsi_{\node,s}}^2,
\end{equation}
with the first being a straightforward calculation.
As such, the CART splitting condition is equivalent to selecting a feature vector from an admissible set that best reduces the residual variance.
The CART update rule adds this feature to the model, and sets its coefficient to minimize the 1-dimensional least squares equation.
Given orthogonality of all the features, this also updates the linear model to the best fit linear model on the new feature set.

Concatenating the decision stumps together yields a feature map $\Psi\colon \R^d \to \R^p$, and we let $\bPsi$ denote the $n$ by $m$ transformed data matrix.
Let $\hat{\bbeta}$ denote the solution to the least squares problem
\begin{equation} \label{eq:least_squares_CART}
     \min_{\bbeta} ~ \norm*{\bPsi\bbeta - \by}_2^2.
\end{equation}
We have just argued that we have functional equality
\begin{equation} \label{eq:CART_as_lin_reg}
    \hat f(\bx) = \hat{\bbeta}^T \Psi(\bx).
\end{equation}
These calculations can be found in more detail in Lemma 3.2 in \cite{klusowski2021universal}.

\subsection{Modifications for \method}
\label{sec:FIGS_linear_supp}

With a collection of trees $\tree_1,\ldots,\tree_K$, we may still associate a normalized decision stump \eqref{eq:decision_stump} to every node and every potential split.
The impurity decrease used to determine splits can still be written as a squared correlation with a potential split vector,
\begin{equation*}
    \hat \Delta(s,\node,r) = \paren{\br^{(-k)T}\bPsi_{\node,s}} ^2.
\end{equation*}
and so the FIGS splitting rule can also be thought of as selecting a feature vector from an admissible set that best reduces the residual variance, while its update rule adds this feature to the model, and sets its coefficient to minimize the 1-dimensional least squares equation.
The difference to CART lies in the fact that the admissible set is now larger, comprising potential splits from multiple trees, and furthermore, the node vectors from different trees are no longer orthogonal to each other, and so the update rule, while solving the 1-dimensional problem, no longer minimizes the full least squares loss.
Nonetheless, as discussed in the main paper, a simple backfitting procedure after the tree structures are fixed is equivalent to block coordinate descent on this linear system, and converges to the minimizer, and furthermore, we observed in our examples that the resulting coefficients do not change too much from their initial values, meaning that the FIGS solution is already close to a best fit linear model.

\subsection{\methods disentangles the additive components of additive generative models}
\label{sec:disentanglement}

\paragraph{Generative model.} 
Let $\bx$ be a random variable with distribution $\pi$ on $[0,1]^d$.
Suppose that we have disjoint blocks of features $I_1,\ldots, I_K$, of sizes $d_1,\ldots,d_K$, with $d = \sum_{k=1}^K d_k$, and suppose the blocks of features are mutually independent, i.e. $\bx_{I_j} \indep \bx_{I_k}$ for $j \neq k$, where for any index set $I$, $\bx_I$ denotes the subvector of $\bx$ comprising coordinates in $I$.
Let $y = f(\bx) + \epsilon$ where $\E\braces*{\epsilon~|~\bx} = 0$ and
\begin{equation} \label{eq:additive}
    f(\bx) = \sum_{k=1}^K f_k(\bx_{I_k}).
\end{equation}
Suppose further that each component function has mean zero with respect to $\pi$.

For a given tree structure $\tree$, let $J(\tree)$ denote the set of features that it splits upon.
We say that a tree-sum model with tree structures $\braces{\tree_1,\ldots,\tree_M}$ completely (respectively partially) \emph{disentangles} the additive components of an additive generative model \eqref{eq:additive} if $M=K$ and, after re-indexing if necessary, $J(\tree_k) = I_k$ (respectively $J(\tree_k) \subset I_k$) for $k=1,\ldots,K$.

In this section, we argue that \methods is able to achieve disentanglement and learn additive components of additive generative models.
To the best of our knowledge, this property is unique to FIGS and is not shared by any other tree ensemble algorithm.
As argued in the introduction, disentanglement helps to avoid duplicate subtrees, leading to a more parsimonious model with better generalization performance (see \cref{thm:generalization_main} for a precise statement) . Note that even when the generative model is not additive, FIGS helps to reduce the number of possibly redundant often observed in CART models (see \cref{sec:sim_results_supp}.)

\begin{theorem}[\textbf{Partial disentanglement in large sample limit}] \label{thm:disentanglement}
    Consider the generative model \eqref{eq:additive}, and recall the notation from \cref{sec:methods}.
    Suppose we run \cref{alg:method} with the following oracle modifications when splitting a node $\node$.\\
    
    \noindent 1. Splits are selected using the population impurity decrease 
      \begin{align} \label{eq:population_impurity}
         \quad \Delta(s,\node, r) &\coloneqq \pi(\node)  \Var\braces*{r~|~\bx \in \node} \\
        & - \pi(\node_L) \Var\braces*{r~|~\bx \in \node_L} 
         - \pi(\node_R) \Var\braces*{r~|~\bx \in \node_R}, \nonumber
        \end{align}
        instead of the finite sample impurity decrease $\hat\Delta(s, \node, \br)$. \\ \\
        2. Values of the new children nodes $\node_L$ and $\node_R$ are obtained by adding, to the value of $\node$, the population means $\E_{\pi}\braces{r~|~\node_L}$ and $\E_{\pi}\braces{r~|~\node_R}$ respectively, instead of the sample means $\bar r_{\node_L}$ and $\bar r_{\node_R}$ respectively. \\
    
    \noindent Then for each fitted tree $\hat f_k$, the set of features split upon is contained within a single index set $I_k$ for some $k$.
\end{theorem}

The proof for this theorem is deferred to \cref{sec:theory_supp}.
Note that for any generative model $h(\bx, y)$, the finite-sample impurity decrease converges to the population impurity decrease as the sample size goes to infinity: $n^{-1}\hat\Delta(s,\node, \bh) \to \Delta(s,\node, h)$ and $\bar h_\node \to \E\braces{h~|~\node}$ as $n \to \infty$.
This justifies our claim that the modified algorithm is equivalent to running \method~in the large sample limit.
Note that the number of terms in the fitted model need not be equal to the number of additive components $K$.

\paragraph{Real-world example.}
We now show how disentanglement could lead to a more scientific accurate and interpretable models for a 
Diabetes classification dataset~\cite{bennett1971diabetes,smith1988using}. In this dataset, eight risk factors were collected and used to predict the onset of diabetes within five years. The dataset consists of 768 female subjects from the Pima Native American population near Phoenix, AZ, USA. 268 of the subjects developed diabetes, which is treated as a binary label.

\cref{fig:model_example} shows two models, one learned by \methods and one learned by CART.
In both models, the prediction corresponds to the risk of diabetes onset withing five years (a higher prediction corresponds to a higher risk).
Both achieve roughly the same performance (\methods yields an AUC of 0.820 whereas CART yields an AUC of 0.817), but the models have some key differences.
The \methods model includes fewer features and fewer total splits than the CART model, making it easier to understand in its entirety.
Moreover, the \methods model has no interactions between features, making it clear that each of the features contributes independently of one another, something which any single-tree model is unable to do.

\begin{figure}[hbtp]
    \centering
    {\textbf{CART} \raggedright\\\vspace{-1pt}}
    \includegraphics[width=0.4\columnwidth]{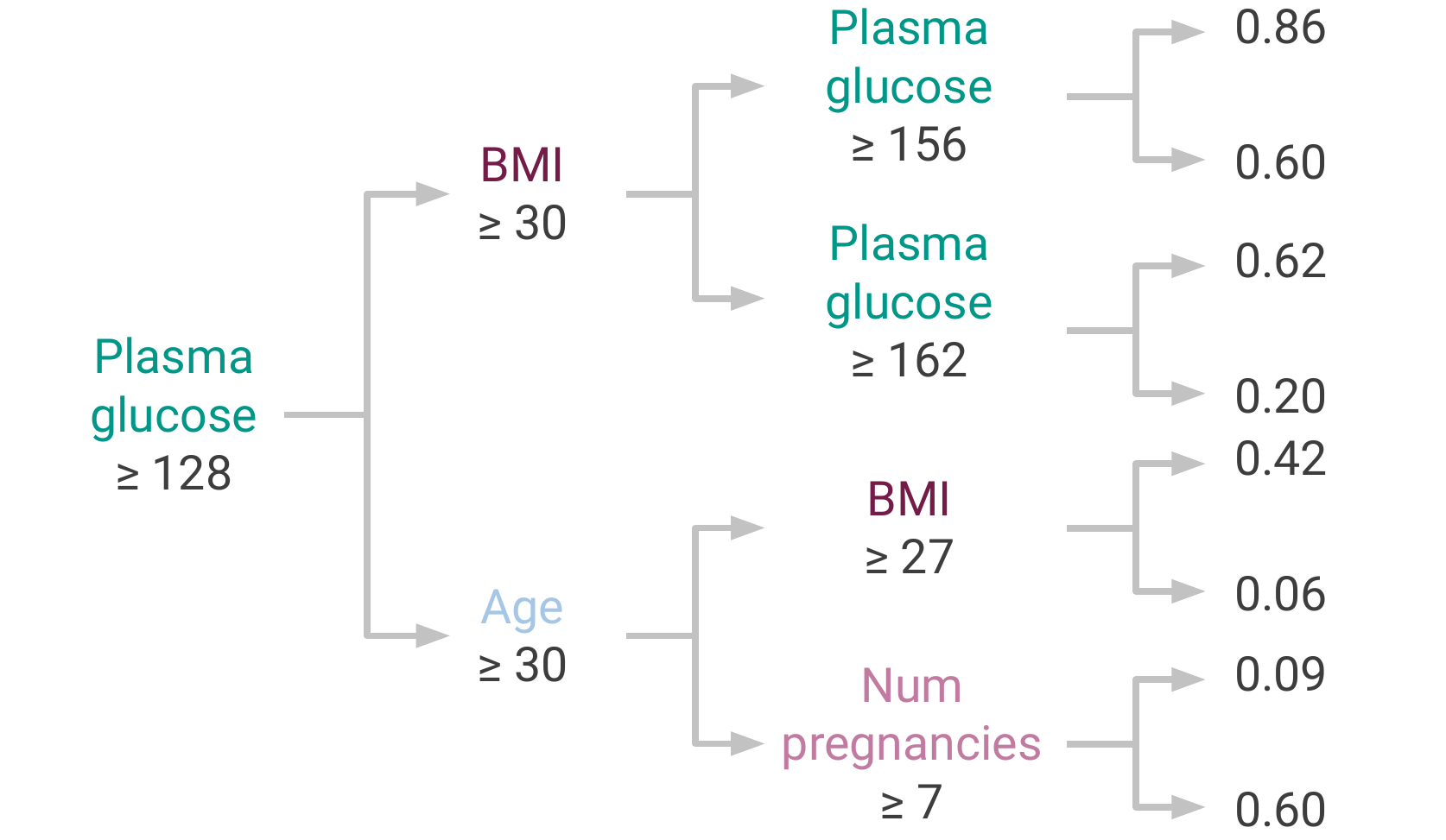}\\
    {\textcolor{lightgray}{ \noindent\makebox[\linewidth]{\rule{0.4\columnwidth}{0.6pt}}}}
    {\textbf{\method} \raggedright\vspace{-5pt}\\}
    \includegraphics[width=0.3\columnwidth]{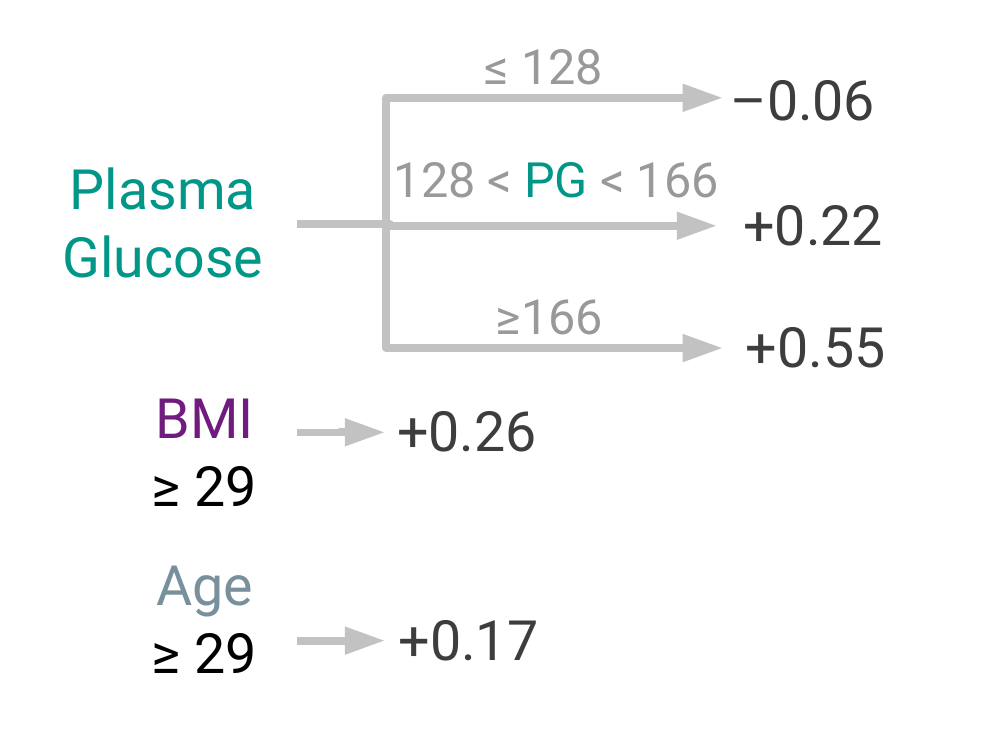}
    \caption{Comparison between \methods and CART on the diabetes dataset. \methods learns a simpler model, which disentangles interactions between features. Both models achieve the same generalization performance (\methods yields an AUC of 0.820 whereas CART yields 0.817.)}
    \label{fig:model_example}
\end{figure}

\subsection{Proof of \cref{thm:disentanglement}}
\label{sec:proof_disentanglement_supp}

\begin{proof}[Proof of \cref{thm:disentanglement}]
    We prove this by induction on the total number of splits, with the base case being trivial.
    By the induction hypothesis, we may assume WLOG that $\hat f_1$ only has splits on features in $I_1$.
    Consider a candidate split $s$ on a leaf $\node \in \hat f_1$ based on a feature $m \in I_2$.
    Let $\node'$ denote the projection of $\node$ onto $I_1$. 
    As sets in $\R^d$, we may then write
    \begin{equation} \label{eq:disentanglement_node_formula}
        \node = \node' \times \R^{[d]\backslash I_1},
    \end{equation}
    \begin{equation} \label{eq:disentanglement_node_left_formula}
        \node_L = \node' \times (-\infty, \tau] \times \R^{[d]\backslash I_1\cup \braces{m}},
    \end{equation}
    and
    \begin{equation} \label{eq:disentanglement_node_right_formula}
        \node_R = \node' \times (\tau, \infty) \times \R^{[d]\backslash I_1\cup \braces{m}}.
    \end{equation}
    
    Recall that we work with the residual $r = f(\bx) - \sum_{k = 1}^K \hat f_k$.
    Now using the law of total variance, we can rewrite the weighted impurity decrease in a more convenient form:
    \begin{equation} \label{eq:disentanglement_imp_dec}
        \Delta(s,\node, r) = \frac{\pi(\node_L)\pi(\node_R)}{\pi(\node)} \paren*{\E\braces*{r~|~\bx \in \node_L} - \E\braces*{r~|~\bx \in \node_R}}^2.
    \end{equation}
    We may assume WLOG that this quantity is strictly positive.
    By the induction hypothesis, we can divide the set of component trees into two collections, one of which only splits on features in $I_2$, and those which only split on features in $[d]\backslash I_2$.
    Denoting the function associated with the second collection of trees by $g_2$, we observe that
    \begin{align*}
        \E\braces*{r~|~\bx \in \node_L} - \E\braces*{r~|~\bx \in \node_R}
        & = \E\braces*{f-g~|~\bx \in \node_L} - \E\braces*{f-g~|~\bx \in \node_R} \\
        & = \E\braces*{f_2 - g_2 ~|~\bx \in \node_L} - \E\braces*{f_2 - g_2 ~|~\bx \in \node_R}.
    \end{align*}
    Since $f_2$ and $g_2$ do not depend on features in $I_1$, we can then further rewrite this quantity as
    \begin{equation} \label{eq:disentanglement_cond_exp}
        \E\braces*{f_2 - g_2 ~|~x_m \leq \tau} - \E\braces*{f_2 - g_2 ~|~x_m > \tau}.
    \end{equation}
    Meanwhile, using \eqref{eq:disentanglement_node_formula}, \eqref{eq:disentanglement_node_left_formula}, and \eqref{eq:disentanglement_node_right_formula}, we may rewrite
    \begin{equation} \label{eq:disentanglement_probs}
        \frac{\pi(\node_L)\pi(\node_R)}{\pi(\node)} = \pi_1(\node')\pi_2(x_m \leq \tau)\pi_2(x_m > \tau).
    \end{equation}
    Plugging \eqref{eq:disentanglement_cond_exp} and \eqref{eq:disentanglement_probs} back into \eqref{eq:disentanglement_imp_dec}, we get
    \begin{equation} \label{eq:disentanglement_imp_dec_smaller}
        \Delta(s,\node, r) = \pi_1(\node')\pi_2(x_m \leq \tau)\pi_2(x_m > \tau) \paren*{\E\braces*{f_2 - g ~|~x_m \leq \tau} - \E\braces*{f_2 - g ~|~x_m > \tau}}^2.
    \end{equation}
    
    In contrast, if we split a new root node $\node_0$ on $m$ at the same threshold and call this split $s'$, we can run through a similar set of calculations to get
    \begin{equation} \label{eq:disentanglement_imp_dec_bigger}
    \Delta(s',\node_0,r) = \pi_2(x_m \leq \tau)\pi_2(x_m > \tau) \paren*{\E\braces*{f_2 - g ~|~x_m \leq \tau} - \E\braces*{f_2 - g ~|~x_m > \tau}}^2.
    \end{equation}
    Comparing \eqref{eq:disentanglement_imp_dec_smaller} and \eqref{eq:disentanglement_imp_dec_bigger}, we see that
    \begin{equation*}
        \Delta(s,\node, r^{(-1)}) = \pi_1(\node')\Delta(s',\node_0,r),
    \end{equation*}
    and as such, split $s'$ will be chosen in favor of $s$.
\end{proof}

\subsection{Theoretical generalization upper bounds}
We shall prove a generalization upper bound showing that disentangled tree-sum models enjoy better prediction performance that single-tree models when fitted to data with additive structure. When the data is generated from a fully additive model with $C^1$ component functions (i.e. the generative model \eqref{eq:additive} with blocks of size $d_k=1$, and $f_k \in C^k([0,1])$ for each $k$), it was recently shown that any single decision tree model has a squared error generalization lower bound of $\Omega(n^{-2/(d+2)})$. This is significantly worse than the minimax rate of $\tilde O\paren*{d n^{-2/3}}$ for this problem~\cite{tan2021cautionary}. Tight generalization upper bounds have proved elusive for CART due to the complexity of analyzing the tree growing procedure, and are difficult for \method~for the same reason.
Nonetheless, we can prove a partial result that shows that a disentangled tree-sum representation is more effective than a single tree model. To formalize this, let $\mathcal{C} = \braces{\tree_1, \ldots, \tree_M}$ be the collection of tree structures learnt by \method. We say that a function is a \emph{tree-sum implementable with $\mathcal C$} if it can be represented as a sum of component functions, each of which is implementable by one of the tree structures $\tree_k$ (i.e. constant on each leaf of $\tree_k$). We use $\mathcal{F}(\mathcal C)$ to denote the collection of all tree-sum implementable with $\mathcal C$.
As discussed in \cref{sec:methods}, \methods essentially fits an empirical risk minimizer from $\mathcal{F}(\mathcal C)$.
In the following theorem, we prove that when $\mathcal C$ is instead chosen by an oracle, an empirical risk minimizer with respect to $\mathcal{F}(\mathcal C)$ has good generalization properties.

\begin{theorem}[\textbf{Generalization upper bounds using oracle tree structures}] \label{thm:generalization_main}
    Consider the generative model \eqref{eq:additive}, and further suppose the distribution $\pi_k$ of each independent block $\bx_{I_k}$ has a continuous density, each $f_k$ is $C^1$, with $\norm{\nabla f_k}_2 \leq \beta_k$,
    and that $\epsilon$ is homoskedastic with variance $\E\braces*{\epsilon^2~|~\bx} = \sigma^2$.
    For any sample size $n$, there exists an oracle collection of trees $\mathcal C = \braces{\tree_1,\ldots,\tree_K}$ disentangling \eqref{eq:additive}
    such that for a training set $\data$ of size $n$, any empirical risk minimizer $\tilde f$ of $\mathcal F (\mathcal C)$ satisfies
    the following squared error generalization upper bound:
    \begin{align} \label{eq:generalization_bound_theorem}
        \E_{\data, \bx} 
        \braces*{(\tilde f(\bx) - f(\bx))^2\indicator
        \braces*{
        \mathcal{E}^c}} \leq \sum_{k=1}^K c_k\paren*{\frac{\sigma^2}{n}}^{\frac{2}{d_k+2}}.
    \end{align}
    Here, $\mathcal E$ is an event with vanishing probability $\P\braces*{\mathcal E} = O(n^{-2/(d_{max}+2)})$ where $d_{max} = \max_k d_k$ is the size of the largest feature block in \eqref{eq:additive}, while $c_k \coloneqq 8\paren*{d_k\beta_k^2\norm{\pi_k}_\infty}^\frac{d_k}{d_k+2}$.\footnote{We note that the error event $\mathcal{E}$ is due to the query point possibly landing in leaf nodes containing very few or even zero training samples, which can be thus be detected and avoided in practice by imputing a default value.}
    
\end{theorem}

The proof of \cref{thm:generalization_main} can be found in \cref{sec:proof_disentanglement_supp}, 
and builds on recent work \cite{klusowski2021universal} which shows how to interpret decision trees as linear models on a set of engineered features corresponding to internal nodes in the tree. This interpretation has a natural extension to \method.
It is instructive to consider two extreme cases: If $d_k=1$ for each $k$, then we have an upper bound of $O\paren*{d n^{-2/3}}$.
If on the other hand $K=1$, we have an upper bound of $O\paren*{n^{-2/(d+2)}}$.
Both (partially oracle) bounds match the well-known minimax rates for their respective inference problems \cite{raskutti2012minimax}.

\subsection{Proof of \cref{thm:generalization_main}}
\label{sec:proof_generelization_supp}

\begin{proof}[Proof of \cref{thm:generalization_main}]
    To construct $\mathcal C$, for each $k$, construct a tree $\tree_k$ that partitions $[0,1]^{d_k}$ into cubes of side length $h_k$, where $h_k$ is a parameter to be determined later.
    Let $p_k$ denote the number of internal nodes in $\tree_k$.
    Recall that we use $\tilde f$ to denote the empirical risk minimizer among all functions in $\mathcal F (\mathcal C)$.
    Let $g$ denote an $\ell_2$ risk minimizer in $\mathcal F (\mathcal C)$.
    For any event $\mathcal{E}$, we then apply Cauchy-Schwarz to get
    \begin{equation} \label{eq:Cauchy_Schwarz}
        \E_{\data, \bx \sim \pi} \braces*{\paren*{\tilde f(\bx) - f(\bx)}^2\indicator\braces*{\mathcal{E}^c}} 
        \leq 2\E\braces*{\paren*{f(\bx)- g(\bx)}^2} + 2\E\braces*{\paren*{g(\bx)- \tilde f(\bx)}^2\indicator\braces*{\mathcal{E}^c}}.
    \end{equation}
    The first term represents bias, and quantifies the error incurred when attempting to approximate $f$ with the function class $\mathcal F(\mathcal C)$, while the second term represents variance, and quantifies the error incurred because of sampling uncertainty.
    We will bound the first term using smoothness properties of the $f_k$'s, while for the second term, we will rewrite $\tilde f$ and $g$ as linear functions in an engineered feature space, which would allow us to use calculations from linear modeling theory to obtain a bound.
    
    We begin by bounding the second term, and assume WLOG that $\E_{\pi_k}\braces*{f_k} = 0$ for each $k$.
    Let $\Psi_k$ be the feature mapping associated with $\tree_k$ as described earlier in this section, and concatenate them for $k=1,\ldots,K$ to get the mapping $\Psi$.
    There is a one-to-one correspondence between $\mathcal F (\mathcal C)$ (tree-sum functions implementable with $\mathcal C$) and linear function with respect to the representation $\Psi$ (i.e. of the form $\bx \mapsto \btheta^T\Psi(\bx)$ for some vector $\btheta$).
    As such, $\tilde f$, the empirical risk minimizer in $\mathcal F (\mathcal C)$, can be written as $\tilde f(\bx) = \tilde \btheta^T \Psi(\bx)$, where $\tilde \btheta$ is the solution to the least squares problem
    \begin{equation*}
        \min_{\btheta} \sum_{i=1}^n \paren*{\btheta^T\Psi(\bx_i) - y_i}^2.
    \end{equation*}
    Furthermore, one can check that the $\ell_2$ risk minimizer $g$ can be written as $g(\bx) = \btheta^{*T}\Psi(\bx)$, where
    \begin{equation} \label{eq:formula_for_theta_star}
        \btheta^*(\node) \coloneqq \frac{\sqrt{N(\node_L)N(\node_R)}}{N(\node)}\paren*{\E\braces*{y~|~\node_L} - \E\braces*{y~|~\node_R}}
    \end{equation}
    for each node $\node$.
    This gives the formula $g(\bx) = \sum_{k=1}^K g_k(\bx_{I_k})$, where for each $k$,
    \begin{equation*}
        g_k(\bx_{I_k}) \coloneqq \E\braces*{f_k(\bx_{I_k}')~|~\bx_{I_k}' \in \node_k(\bx_{I_k})}.
    \end{equation*}
    Here, $\node_k(\bx_{I_k})$ is the leaf in $\tree_k$ containing $\bx_{I_k}$, and $\bx_{I_k}'$ an independent copy of $\bx_{I_k}$.

    Meanwhile, note that we have the equation
    $$
    y = \btheta^{*T}\Psi(\bx) + \eta + \epsilon
    $$
    where $\eta \coloneqq f(\bx) - g(\bx)$ satisfies
    \begin{align*}
        \E\braces*{\eta~|~\Psi(\bx)}
        & = \E\braces*{\sum_{k=1}^K \paren*{f_k(\bx_{I_k}) - g_k(\bx_{I_k})}~|~\Psi(\bx)} \\
        & = \sum_{k=1}^K\E\braces*{ f_k(\bx_{I_k}) - g_k(\bx_{I_k})~|~\Psi_k(\bx_{I_k})} \\
        & = 0.
    \end{align*}
    As such, we may apply Theorem \ref{thm:generalization_lin_reg} with the event $\mathcal E$ given in the statement of the theorem to get
    \begin{align*}
        \E\braces*{\paren*{g(\bx)- \tilde f(\bx)}^2\indicator\braces*{\mathcal E ^c}} \leq 2\paren*{\frac{p\sigma^2}{n+1} + 2\E\braces*{\paren*{f(\bx)- g(\bx)}^2}}.
    \end{align*}
    Plugging this into \eqref{eq:Cauchy_Schwarz}, we get
    \begin{equation} \label{eq:continuation_of_CS_pt1}
        \E_{\data, \bx \sim \pi} \braces*{\paren*{\tilde f(\bx) - f(\bx)}^2\indicator\braces*{\mathcal{E}^c}} 
        \leq 10\E\braces*{\paren*{f(\bx)- g(\bx)}^2} + \frac{4p\sigma^2}{n+1}.
    \end{equation}
    By independence, and the fact that $\E\braces*{g_k(\bx_{I_k})} = 0$ for each $k$, we can decompose the first term as
    \begin{equation*}
        \E\braces*{\paren*{f(\bx)- g(\bx)}^2} = \sum_{k=1}^K \E\braces*{\paren*{f_k(\bx)- g_k(\bx)}^2}.
    \end{equation*}
    This allows us to decompose the right hand side of \eqref{eq:continuation_of_CS_pt1} into the sum
    \begin{equation} \label{eq:continuation_of_CS}
        \sum_{k=1}^K \paren*{10\E\braces*{\paren*{f_k(\bx_{I_k})- g_k(\bx_{I_k})}^2} + \frac{4p_k\sigma^2}{n+1}}.
    \end{equation}
    
    We reduce to the case of uniform distribution $\mu$, via the inequality
    $$
    \E_{\pi_k}\braces*{\paren*{f_k(\bx_{I_k})- g_k(\bx_{I_k})}^2} \leq \norm{\pi_k}_\infty\E_\mu\braces*{\paren*{f_k(\bx_{I_k})- g_k(\bx_{I_k})}^2},
    $$
    and from now work with this distribution, dropping the subscript for conciseness.
    Next, observe that
    \begin{equation*}
        \E\braces*{\paren*{f_k(\bx_{I_k})- g_k(\bx_{I_k})}^2} = \E\braces*{\Var\braces*{f_k(\bx_{I_k})~|~\node_k(\bx_{I_k})}}.
    \end{equation*}
    Using Lemma \ref{lem:variance_and_sides}, we have that
    \begin{equation*}
        \Var\braces*{f_k(\bx_{I_k})~|~\node_k(\bx_{I_k})} 
        \leq \frac{\beta_k^2d_k h_k^2}{6}.
    \end{equation*}
    Meanwhile, a volumetric argument gives
    \begin{equation*}
        p_k \leq h_k^{-d_k}.
    \end{equation*}
    We use these to bound each term of \eqref{eq:continuation_of_CS} as
    \begin{equation} \label{eq:continuation_of_CS_pt2}
        10\norm{\pi_k}_\infty\E\braces*{\paren*{f_k(\bx_{I_k})- g_k(\bx_{I_k})}^2} + \frac{4p_k\sigma^2}{n+1} 
        \leq 2\norm{\pi_k}_\infty\beta_k^2d_k h_k^2 + \frac{4h_k^{-d_k}\sigma^2}{n+1}.
    \end{equation}
    Pick
    $$
    h_k = \paren*{\frac{2\sigma^2}{\norm{\pi_k}_\infty\beta_k^2d_k(n+1)}}^{\frac{1}{d_k+2}},
    $$
    which sets both terms on the right hand side to be equal, in which case the right hand of \eqref{eq:continuation_of_CS_pt2} has the value
    \begin{equation*}
        4\paren*{2\norm{\pi_k}_\infty\beta_k^2d_k}^{\frac{d_k}{d_k+2}}\paren*{\frac{\sigma^2}{n+1}}^{\frac{2}{d_k+2}}.
    \end{equation*}
    Summing these quantities up over all $k$ gives the bound \eqref{eq:generalization_bound_theorem}, with the error probability obtained by computing $2p/n$.
\end{proof}

\begin{corollary} \label{cor:sparse_additive}
    Assume a sparse additive model, i.e. in \eqref{eq:additive}, assume $I_k = \braces{k}$ for $k=1,\ldots,K$. Then we have
    $$
    \E_{\data, \bx \sim \pi} \braces*{\paren*{\tilde f(\bx) - f(\bx)}^2\indicator\braces*{\mathcal{E}^c}}
    \leq 8K\max_k \paren*{\norm{\pi_k}_\infty\beta_k^2}^{1/3}\paren*{\frac{\sigma^2}{n}}^{\frac{2}{3}}.
    $$
\end{corollary}

\begin{lemma}[Variance and side lengths] \label{lem:variance_and_sides}
    Let $\uniform$ be the uniform measure on $[0,1]^d$. 
    Let $\cell \subset [0,1]^d$ be a cell. Let $f$ be any differentiable function such that $\norm{\grad f(\bx)}_2^2 \leq \beta^2$.
    Then we have
    \begin{equation} \label{eq:variance_and_sides}
        \Var_\mu\braces*{f(\bx)~|~\bx \in \cell} \leq \frac{\beta^2}{6}\sum_{j=1}^d (b_j-a_j)^2.
    \end{equation}
\end{lemma}

\begin{proof}
    For any $\bx, \bx' \in \cell$, we may write
    $$
    \paren*{f(\bx) - f(\bx')}^2 = \inprod*{\grad f(\bx''),\bx - \bx'}^2 \leq \beta^2 \norm{\bx-\bx'}_2^2.
    $$
    Next, note that
    $$
    \E\braces*{\norm{\bx-\bx'}_2^2~|~\bx,\bx' \in \cell}^2 = \frac{1}{3}\sum_{j=1}^d (b_j-a_j)^2.
    $$
    As such, we have
    \begin{align*}
        \Var_\mu\braces*{f(\bx)~|~\bx \in \cell} & = \frac{1}{2}\E\braces*{\paren*{f(\bx) - f(\bx')}^2~|~ \bx, \bx' \in c\ell} \\
        & \leq \frac{\beta^2}{6}\sum_{j=1}^d (b_j-a_j)^2.
    \end{align*}
\end{proof}

\subsection{Helper lemmas on linear regression}
\label{sec:helper_lemmas_supp}

We consider the case of possibly under-determined least squares, i.e. the problem
\begin{equation} \label{eq:least_squares}
     \min_{\btheta} ~ \norm*{\bX\btheta - \by}_2^2
\end{equation}
where we allow for the possibility that $\bX$ does not have linearly independent columns.
When this is indeed the case, there will be multiple solutions to \eqref{eq:least_squares}, but there is a unique element $\hat{\btheta}$ of the solution set that has minimum norm.
In fact, this is given by the formula
$$
\hat{\btheta} = \bX^{\dagger}\by,
$$
where $\bX^\dagger$ denotes the Moore-Penrose pseudo-inverse of $\bX$.

We extend the definition of leverage scores to this case by defining the $i$-th leverage score $h_i$ to be the $i$-th diagonal entry of the matrix $\bH \coloneqq \bX\bX^\dagger$.
Note that the vector of predicted values is given by we have
$$
\hat\by = \bX\hat\btheta = \bX\bX^{\dagger}\by = \bH\by,
$$
so that this coincides with the definition of leverage scores in the linearly independent case.

In what follows, we will work extensively with leave-one-out (LOO) perturbations of the sample and the resulting estimators.
We shall use $\bX^{(-i)}$ to denote the data matrix with the $i$-th data point removed, and $\hat \btheta^{(-i)}$ to denote the solution to \eqref{eq:least_squares} with $\bX$ replaced with $\bX^{(-i)}$.
We have the following two generalizations of standard formulas in the full rank setting.

\begin{lemma}[Leave-one-out estimated coefficients] \label{lem:LOO_coefficients}
    The LOO estimated coefficients satisfy
    \begin{equation*}
        \bx_i^{T}\paren*{\hat\btheta - \hat\btheta^{(-i)}} = \frac{h_i\hat e_i}{1-h_i}
    \end{equation*}
    where $\hat e_i = y_i - \bx_i^T\hat\btheta$ is the residual from the full model.
\end{lemma}

\begin{proof}
    Note that we may write $\bX^\dagger = \paren*{\bX^T\bX}^\dagger\bX^T$.
    We may then follow the proof of the usual identity but substituting \eqref{eq:sherman_morrison} in lieu of the regular Sherman-Morrison formula.
\end{proof}

\begin{lemma}[Sherman-Morrison]
    Let $\bX$ be any matrix. Then
    \begin{equation} \label{eq:sherman_morrison}
        \paren*{\bX^{(-i)T}\bX^{(-i)}}^\dagger = \paren*{\bX^T\bX}^{\dagger} + \frac{\paren*{\bX^T\bX}^{\dagger}\bx_i\bx_i^T\paren*{\bX^T\bX}^{\dagger}}{1-h_i}.
    \end{equation}
\end{lemma}

\begin{proof}
    We apply Theorem 3 in \cite{meyer1973generalized} to $A = \bX^T\bX$, $c = \bx_i$ and $d = - \bx_i$, noting that the necessary conditions are fulfilled.
\end{proof}

\begin{theorem}[Generalization error for linear regression] \label{thm:generalization_lin_reg}
    Consider a linear regression model
    $$
    y = \btheta^T\bx + \epsilon + \eta
    $$
    $\epsilon$ and $\eta$ are independent, $\E\braces*{\epsilon~|~\bx} = \E\braces*{\eta~|~\bx} = 0$, $\E\braces*{\epsilon^2~|~\bx} = \sigma_\epsilon^2$, and $\E\braces*{\eta^2} = \sigma_\eta^2$.
    Given a training set $\data$, and a query point $\bx$, let $\hat\btheta_n$ be the estimated regression vector.
    There is an event $\mathcal{E}$ of probability at most $2p/n$ over which we have
    \begin{equation}
        \E_{\data,\bx}\braces*{\paren*{\bx^T\paren*{\hat\btheta_{n} - \btheta}}^2\indicator\braces*{\mathcal E ^c}}
        \leq 2\paren*{\frac{p\sigma_\epsilon^2}{n+1} + 2\sigma_\eta^2}.
    \end{equation}
\end{theorem}

\begin{proof}
    Let $\data$ denote the training set. Let $\mathcal{E}$ be the event on which $\bx^T\paren*{\bX^T\bX}^{\dagger}\bx \geq \frac{1}{2}$. 
    This quantity is the leverage score for $\bx$, so that
    $$
    \E\braces*{\bx^T\paren*{\bX^T\bX}^{\dagger}\bx} \leq \frac{p}{n+1}
    $$
    by exchangeability of $\bx$ with the data points in $\data$.
    We may then apply Markov's inequality to get
    $$
    \P\braces*{\mathcal E} \leq \frac{2p}{n+1}.
    $$
    
    Again using exchangeability, we may write
    \begin{align} \label{eq:risk_decomposition}
        \E_{\data,\bx}\braces*{\paren*{\bx^T\paren*{\hat{\btheta}_n - \btheta}}^2\indicator\braces*{\mathcal{E}^c}} 
        = \frac{1}{n+1} \sum_{i=0}^n\E_{\data[n+1]} \braces*{\paren*{\bx_i^T\paren*{\hat{\btheta}_{n+1}^{(-i)} - \btheta}}^2\indicator\braces*{\mathcal{E}_i^c}},
    \end{align}
    where $\data[n+1]$ is the augmentation of $\data$ with the query point $\bx_0 = \bx$ and response $y_0$, $\hat{\btheta}_{n+1}$ is the regression vector learnt from $\data[n+1]$, and for each $i$, $\hat{\btheta}_{n+1}^{(-i)}$ that from $\data[n+1]\backslash \braces*{\paren*{\bx_i,y_i}}$.

    To bound this, we first rewrite the prediction error for the full model as
    \begin{align} \label{eq:full_model_pred_error}
        \bx_i^T\paren*{\hat\btheta_{n+1} -\btheta} 
        & = \bx_i^T\bX^\dagger \by - \bx_i^T\btheta \nonumber \\
        & = \bx_i^T\bX^\dagger \paren*{\bX\btheta + \beps + \beeta} - \bx_i^T\btheta \nonumber \\
        & = \bx_i^T\bX^\dagger\paren*{\beps + \beeta},
    \end{align}
    where the last equality follows because $\bx_i$ lies in the column space of $\bX^\dagger\bX$.
    Next, we may decompose that for the LOO model as
    \begin{equation} \label{eq:LOO_model_pred_error}
        \bx_i^T\paren*{\hat{\btheta}_{n+1}^{(-i)} - \btheta} 
        = \bx_i^T\paren*{\hat{\btheta}_{n+1}^{(-i)} - \hat\btheta_{n+1}} + \bx_i^T\paren*{\hat\btheta_{n+1} - \btheta}.
    \end{equation}
    We expand the first term using Lemma \ref{lem:LOO_coefficients} to get
    \begin{align} \label{eq:full_and_LOO_diff}
        \bx_i^T\paren*{\hat{\btheta}_{n+1}^{(-i)} - \hat\btheta_{n+1}}
        & = \frac{h_i}{1-h_i}\paren*{\bx_i^T\hat\btheta_{n+1} - y_i} \nonumber \\
        & = \frac{h_i}{1-h_i}\paren*{\bx_i^T\paren*{\hat\btheta_{n+1} - \btheta} - \epsilon_i - \eta_i}.
    \end{align}
    We plug \eqref{eq:full_and_LOO_diff} into \eqref{eq:LOO_model_pred_error} and then \eqref{eq:full_model_pred_error} into the resulting equation to get
    \begin{align*}
        \bx_i^T\paren*{\hat{\btheta}_{n+1}^{(-i)} - \btheta} 
        & = \frac{h_i}{1-h_i}\paren*{\bx_i^T\paren*{\hat\btheta_{n+1} - \btheta} - \epsilon_i - \eta_i} + \bx_i^T\paren*{\hat\btheta_{n+1} - \btheta} \\
        & = \frac{1}{1-h_i} \paren*{\bx_i^T\paren*{\hat\btheta_{n+1} - \btheta}} - \frac{h_i}{1-h_i}\paren*{\epsilon_i + \eta_i} \\
        & = \frac{1}{1-h_i} \bx_i^T\bX^\dagger\paren*{\beps + \beeta} - \frac{h_i}{1-h_i}\paren*{\epsilon_i + \eta_i}.
    \end{align*}
    Taking expectations and using the independence of $\beps$ and $\beeta$, we get
    \begin{align*}
        \E \braces*{\paren*{\bx_i^T\paren*{\hat{\btheta}_{n+1}^{(-i)} - \btheta}}^2\indicator\braces*{\mathcal{E}_i^c}} 
        & = \E\braces*{\paren*{\frac{\bx_i^T\bX^\dagger\beps - h_i\epsilon_i}{1-h_i}}^2\indicator\braces*{\mathcal{E}_i^c}} 
        + \E\braces*{\paren*{\frac{\bx_i^T\bX^\dagger\beeta - h_i\eta_i}{1-h_i}}^2\indicator\braces*{\mathcal{E}_i^c}} \nonumber\\
        & \leq \E\braces*{\paren*{\frac{\bx_i^T\bX^\dagger\beps - h_i\epsilon_i}{1-h_i\wedge 1/2}}^2} 
        + \E\braces*{\paren*{\frac{\bx_i^T\bX^\dagger\beeta - h_i\eta_i}{1-h_i \wedge 1/2}}^2}.
    \end{align*}
    Summing up the first term over all indices, and then taking an inner expectation with respect to $\beps$, we get
    \begin{align}  \label{eq:sq_error_first_term}
        \frac{1}{n+1}\sum_{i=0}^n \E\braces*{\paren*{\frac{\bx_i^T\bX^\dagger\beps - h_i\epsilon_i}{1-h_i\wedge 1/2}}^2} 
        & = \frac{1}{n+1}\E\braces*{\norm*{\paren{1-\diag\paren*{\bH}\wedge 1/2}^{-1}\paren*{\bH - \diag\paren*{\bH}}\beps}_2^2} \nonumber\\
        & = \frac{\sigma_\epsilon^2}{n+1}\E\braces*{\trace\paren*{\bW\bW^T}}.
    \end{align}
    where
    $$
    \bW = \paren{1-\diag\paren*{\bH}\wedge 1/2}^{-1}\paren*{\bH - \diag\paren*{\bH}}.
    $$
    Since $\bH$ is idempotent, we get
    $$
    \paren*{\bH - \diag\paren*{\bH}}^2 = \bH - \bH\diag\paren*{\bH} - \diag\paren*{\bH}\bH + \diag\paren*{\bH}^2.
    $$
    The $i$-th summand in the trace therefore satisfies
    \begin{align*}
        \paren*{\bW\bW^T}_{ii} & = \frac{h_i - 2h_i^2 + h_i^2}{\paren*{1 - h_i\wedge 1/2}^2} \\
        & \leq \frac{h_i}{1 - h_i\wedge 1/2} \\
        & \leq 2h_i.
    \end{align*}
    Summing these up, we therefore continue \eqref{eq:sq_error_first_term} to get
    \begin{equation} \label{eq:sq_err_first_term_final}
        \frac{\sigma_\epsilon^2}{n+1}\E\braces*{\trace\paren*{\bW\bW^T}} \leq \frac{2\sigma_\epsilon^2}{n+1}\E\braces*{\trace\paren*{\bH}} \leq \frac{2p\sigma_\epsilon^2}{n+1}.
    \end{equation}
    
    Next, for any $\bx$, we slightly abuse notation, and denote $\sigma_\eta^2(\bx) = \E\braces*{\eta^2~|~\bx}$.
    By a similar calculation, we get
    $$
    \frac{1}{n+1}\sum_{i=0}^n \E\braces*{\paren*{\frac{\bx_i^T\bX^\dagger\beeta - h_i\eta}{1-h_i\wedge 1/2}}^2} 
    = \frac{1}{n+1}\E\braces*{\trace\paren*{\bW\bSigma\bW^T}},
    $$
    where $\bSigma$ is a diagonal matrix with entries given by $\bSigma_{ii} = \sigma_\eta^2(\bx_i)$ for each $i$.
    We compute
    \begin{align*}
        \paren*{\bW\bSigma\bW^T}_{ii} 
        = \frac{\sum_{j \neq i} \bH_{ij}^2\sigma_\eta(\bx_j)^2}{\paren*{1-h_i \wedge 1/2}^2} 
        \leq 4\sum_{j=0}^n \bH_{ij}^2\sigma_\eta(\bx_j)^2 
        = 4\paren*{\bH\bSigma\bH^T}_{ii}.
    \end{align*}
    This implies that
    \begin{align} \label{eq:sq_err_second_term_final}
        \frac{1}{n+1}\E\braces*{\trace\paren*{\bW\bSigma\bW^T}} 
        & \leq \frac{4}{n+1}\E\braces*{\trace\paren*{\bH\bSigma\bH^T}} \nonumber\\
        & \leq \frac{4}{n+1}\E\braces*{\trace\paren*{\bSigma}} \nonumber\\
        & = 4\sigma_\eta^2.
    \end{align}
    Applying \eqref{eq:sq_err_first_term_final} and \eqref{eq:sq_err_second_term_final} into \eqref{eq:risk_decomposition} completes the proof.
\end{proof}

\begin{remark}
    Note that while we have bounded the probability of $\mathcal{E}$ by $\frac{2p}{n}$, it could be much smaller in value.
    If $\Psi$ is constructed out of a single tree, then $\mathcal{E}$ holds if and only if the test point lands in a leaf containing no training points.
    \cite{tan2021cautionary} shows that the probability of this event decays exponentially in $n$.
\end{remark}

\end{document}